\documentclass{article}

\usepackage{microtype}
\usepackage{graphicx}
\usepackage{subcaption}
\usepackage{booktabs} 
\usepackage{multirow}
\usepackage{orcidlink}

\usepackage{hyperref}

\usepackage{algorithmic}

 \usepackage[accepted]{icml2026}

\usepackage{amsmath}
\usepackage{amssymb}
\usepackage{mathtools}
\usepackage{amsthm}

\usepackage[capitalize,noabbrev]{cleveref}

\theoremstyle{plain}
\newtheorem{theorem}{Theorem}[section]
\newtheorem{proposition}[theorem]{Proposition}

\newtheorem{corollary}[theorem]{Corollary}
\theoremstyle{definition}

\newtheorem{assumption}[theorem]{Assumption}
\theoremstyle{remark}

\usepackage[textsize=tiny]{todonotes}

\icmltitlerunning{Adaptive Neuron-level Mixed Precision QAT}

\begin{document}

\twocolumn[
  \icmltitle{Scale When Needed: Adaptive Neuron-level Mixed Precision\\Quantization Aware Training}



  \icmlsetsymbol{equal}{*}

  \begin{icmlauthorlist}
    \icmlauthor{Ayush K. Varshney\orcidlink{0000-0002-8073-6784}}{yyy,comp}
    \icmlauthor{Konstantinos Vandikas\orcidlink{0000-0001-6925-0954}}{yyy}
    \icmlauthor{Šarūnas Girdzijauskas\orcidlink{0000-0003-4516-7317}}{comp,sch}\\
    \icmlauthor{Adam Orucu\orcidlink{0000-0001-7324-7184}}{yyy,comp}
    \icmlauthor{Aneta Vulgarakis Feljan\orcidlink{0009-0001-1950-7739}}{yyy}
  \end{icmlauthorlist}

  \icmlaffiliation{yyy}{Ericsson Research, Ericsson AB, Stockholm, Sweden}
  \icmlaffiliation{comp}{KTH Royal Institute of Technology, Stockholm, Sweden}
  \icmlaffiliation{sch}{RISE Research Institutes of Sweden, Stockholm, Sweden.}

  \icmlcorrespondingauthor{Ayush K. Varshney}{ayush.kumar.varshney@ericsson.com}

  \icmlkeywords{Machine Learning, ICML}

  \vskip 0.3in
]



\printAffiliationsAndNotice{}  

\begin{abstract}
Deploying deep neural networks on resource-constrained 6G edge devices demands aggressive compression with minimal accuracy loss. Quantization-Aware Training (QAT) has emerged as a leading compression approach; however, existing mixed-precision methods typically operate at coarse layer- or channel-level granularity. These methods often rely on heuristic or search-based bit-allocation strategies, which may overlook fine-grained variability at the neuron level. We propose Neuron-Level Mixed-Precision QAT (NMP-QAT), where each neuron independently learns its own discrete precision during training. Starting from low-bit precision, NMP-QAT expands bit-width only when training signals demand it, via differentiable surrogates and straight-through estimators, while preserving a fully discrete inference graph. This adaptability extends to both weights and activations, reducing memory movement. Evaluated on telecom and non-telecom datasets across MLP and tabular foundation model architectures, NMP-QAT achieves superior compression-accuracy trade-offs over mixed-precision QAT baselines, making it well-suited for Green AI deployments at the network edge.
\end{abstract}

\section{Introduction}

6G edge deployments demand sustainable, efficient AI-native intelligence on resource-constrained hardware~\cite{mao2024green}, making model compression a first-order design constraint, particularly in Global South settings where power, cost, connectivity, and hardware limitations are more stringent. Quantization reduces memory footprint, compute intensity, and data movement by replacing floating-point representations with low-precision fixed-point ones~\cite{aswale2025optimizing}. Quantization-Aware Training (QAT)~\cite{jacob2018quantization} integrates compression into training, outperforming post-training approaches~\cite{zhang2023post}. Uniform fixed-precision QAT~\cite{parkuniform} is hardware-friendly but ignores varying sensitivity across model components. Mixed-precision methods~\cite{rakka2024review} improve efficiency yet operate at layer or channel granularity~\cite{chen2024channel}, missing neuron-level variability. Bit-width search via heuristics~\cite{fang2025mixed} or reinforcement learning~\cite{elthakeb2018releq} adds overhead, limiting scalability. Outlier-aware quantization~\cite{lee2024owq} addresses extreme values but not neuron-level precision heterogeneity. Activation quantization is equally critical: activations dominate on-chip memory traffic yet exhibit larger dynamic ranges, making them harder to quantize without accuracy loss. A detailed related work survey is given in Appendix~\ref{sec:related_works}.

We introduce \emph{Neuron-Level Mixed-Precision QAT (NMP-QAT)}, a framework in which each neuron learns its own discrete precision during training. NMP-QAT initializes at low-bit precision and expands bit-width only when gradient signals demand it, using differentiable surrogates and straight-through estimators in the backward pass while maintaining a fully discrete, hardware-deployable forward graph. This adaptive mechanism extends jointly to weights and activations, enabling fine-grained per-neuron precision control without manual search or multi-stage pipelines.

\noindent\textbf{Contributions:}
\begin{itemize}
    \item \textbf{Neuron-level precision learning.} NMP-QAT starts at low-bit precision and expands bit-widths only when training demands it, achieving aggressive compression with minimal accuracy loss.
    \item \textbf{Unified weights and activations support.} A training paradigm which covers weights-only and weights\,+\,activations quantization with per-neuron activation ranges.
    \item \textbf{Empirical validation.} Evaluated on telecom and general-purpose benchmarks with MLP and TabFormer architectures, NMP-QAT outperforms mixed-precision QAT baselines on accuracy-efficiency trade-offs (Sec.~\ref{sec:exp_details}).
\end{itemize}

\section{Proposed Method}
\label{sec:method}

NMP-QAT enables each neuron to learn its own discrete quantization
precision during training. Neurons within the same layer exhibit
heterogeneous sensitivity to quantization noise. This can arise from
differences in activation distributions, weight magnitudes, and
downstream influence on the output, making a single layer-wise
precision assignment simultaneously over-precise for some neurons and
under-precise for others. NMP-QAT resolves this by learning per-neuron
precision for both weights and activations jointly, as activations
dominate memory traffic and data movement during inference yet are
harder to quantize due to their larger dynamic range.

The framework supports two configurations: \textbf{weights-only}
quantization, where each neuron learns the precision of its incoming
weights; and \textbf{weights\,+\,activations} quantization, where each
neuron additionally learns its activation precision and a
neuron-specific clipping range. Both share the same optimization
principles. We present NMP-QAT for MLPs, where neuron-level
quantization aligns naturally with dense layers, though the approach
extends to convolutional and attention-based architectures.

\subsection{Neuron-Level Weight Quantization}

Consider a fully connected layer with weight matrix
$\mathbf{W} \in \mathbb{R}^{d_{\mathrm{in}} \times d_{\mathrm{out}}}$,
where $\mathbf{W}_{:,j}$ denotes the incoming weights of neuron $j$.
Each neuron $j$ has a trainable \emph{precision strength}
$s_j \in [0,1]$, mapped to a discrete bit-width via fixed thresholds
$0 < t_1 < t_2 < t_3 < 1$:
\begin{equation} \label{eq:bit_threshold}
b(s_j) =
\begin{cases}
1, & s_j \le t_1, \\
2, & t_1 < s_j \le t_2, \\
4, & t_2 < s_j \le t_3, \\
8, & s_j > t_3,
\end{cases}
\end{equation}
over candidate precisions $\mathcal{B}_w = \{1, 2, 4, 8\}$ bits.
To enforce low-precision initialization, $s_j$ is set in $[0, t_1)$
at the start of training. The forward-pass quantized weights are:
\begin{equation}
\mathbf{W}^{(q)}_{:,j} = \mathcal{Q}^{(b(s_j))}(\mathbf{W}_{:,j}),
\end{equation}
where $\mathcal{Q}^{(b)}(\cdot)$ is a deterministic uniform quantizer
at bit-width $b$.

\begin{algorithm}[tb]
  \caption{NMP-QAT (Weights or Weights\,+\,Activations)}
  \label{alg:nmp-qat}
  \begin{algorithmic}
    \STATE {\bfseries Input:} Data $\mathcal{D}$, thresholds $\{t_k\}$
    \STATE Initialize $\mathbf{W}$, $\{s_j\}$; if activation quantization
           enabled, initialize $\{s^{(a)}_j\}$ and $\{\alpha_j\}$
    \REPEAT
    \STATE {\bfseries Forward:} $b_j \leftarrow b(s_j)$; quantize
           weights via $\mathcal{Q}^{(b_j)}$
    \IF{activation quantization enabled}
    \STATE $b^{(a)}_j \leftarrow b(s^{(a)}_j)$; quantize activations
           via $\mathcal{Q}^{(b^{(a)}_j)}$
    \ENDIF
    \STATE Compute task loss $\mathcal{L}$
    \STATE {\bfseries Backward:} Apply STE for quantizers; soft
           surrogates for $\{s_j\}$, $\{s^{(a)}_j\}$; exact gradient
           for $\{\alpha_j\}$
    \STATE Update all parameters with optimizer
    \UNTIL{training complete}
    \STATE {\bfseries Post-training:} Remove surrogates; fix discrete
           precision per neuron
    \STATE {\bfseries Output:} Static mixed-precision inference model
  \end{algorithmic}
\end{algorithm}

\paragraph{Differentiable Precision Learning.}
Since $b(s_j)$ is piecewise-constant and non-differentiable, we
introduce a soft surrogate exclusively in the backward pass:
\begin{equation} \label{eq:weight_surrogate}
\tilde{\mathbf{W}}_{:,j} =
  \sum_{b \in \mathcal{B}_w} w_{j,b}\,\mathcal{Q}^{(b)}(\mathbf{W}_{:,j}),
\end{equation}
where $\mathbf{w}_j = [w_{j,1}, w_{j,2}, w_{j,4}, w_{j,8}]$ are
smooth gate weights derived from sigmoid transitions over $s_j$.
Gradients with respect to the weights themselves use a
straight-through estimator (STE),
$\partial \mathcal{Q}^{(b)}(x)/\partial x \approx 1$,
while gradients with respect to $s_j$ flow through the surrogate
mixture~\eqref{eq:weight_surrogate}. The forward pass \emph{always}
uses the hard discrete assignment~\eqref{eq:bit_threshold},
preserving a fully discrete inference graph.

\subsection{Neuron-Level Activation Quantization}
\label{sec:neuron-act}

When activation quantization is enabled, each neuron $j$ additionally
learns an activation precision strength $s^{(a)}_j \in [0,1]$ and a
learnable clipping range $\alpha_j > 0$. Because activations
post-nonlinearity exhibit large dynamic ranges, binary and ternary
precisions cause severe accuracy degradation; we therefore restrict
activation candidates to:
\begin{equation}
\mathcal{B}_a = \{4, 8, 16\}\ \text{bits},
\end{equation}
with thresholds $0 < t_1 < t_2 < 1$ mapping $s^{(a)}_j$ to
$b(s^{(a)}_j) \in \mathcal{B}_a$ analogously to
\eqref{eq:bit_threshold}. The quantized activation of neuron $j$ for
sample $i$ is:
\begin{equation}
z^{(q)}_{i,j} =
  \mathcal{Q}^{(b(s^{(a)}_j))}\!\big(\mathrm{clip}(z_{i,j},\,0,\,\alpha_j)\big).
\end{equation}
Gradients are computed as follows: STE is applied to
$\mathcal{Q}^{(b)}(\cdot)$, the soft surrogate~\eqref{eq:weight_surrogate}
(with $\mathcal{B}_a$) is used for $s^{(a)}_j$, and exact gradients
are used for $\alpha_j$.

After training, all surrogate constructs are discarded. Each neuron
retains a single fixed discrete precision for weights and (if enabled)
activations, yielding a static mixed-precision inference graph that
exactly reproduces the hard-quantized forward pass used during training.

\begin{table*}[!h]
\centering
\caption{Performance comparison across datasets for weights and activations quantization.}
\label{tab:main_results_act}
\scriptsize
\setlength{\tabcolsep}{3pt}
\renewcommand{\arraystretch}{1.4}
\begin{tabular}{c|c|cccc|c|c|c}
\toprule
\multirow{2}{*}{Dataset} & \multirow{2}{*}{Full Precision} & \multicolumn{4}{c|}{QAT} & \multirow{2}{*}{PPSO} & \multirow{2}{*}{OWQ} & \multirow{2}{*}{Ours} \\
\cline{3-6}
& & {1-bit} & {1.58-bit} & {4-bit} & {8-bit} & & & \\
\midrule
KVS (mse)$\downarrow$ & 5.379 & 3.68 $\pm$ 1.045 & 3.67 $\pm$ 1.084 & 8.696 $\pm$ 1.173 & 8.708 $\pm$ 1.231 & 3.303 $\pm$ 0.122 & 3.672 $\pm$ 0.871 & \textbf{2.795 $\pm$ 0.286} \\ \hline
VoD (mse)$\downarrow$ & \textbf{16.41} & 17.088 $\pm$ 0.617 & 17.090 $\pm$ 0.68 & 26.69 $\pm$ 0.191 & 22.573 $\pm$ 4.116  & 17.429 $\pm$ 0.094 & 18.771 $\pm$ 0.474 & \textit{16.969 $\pm$ 0.117} \\ \hline
RSS (mse)$\downarrow$ & \textbf{17.62} & 36.675 $\pm$ 0.378 & 34.979 $\pm$ 0.504 & 230.981 $\pm$ 0.702 & 230.947 $\pm$ 0.459 & 35.520 $\pm$ 0.551 & 55.234 $\pm$ 4.602 & \textit{33.675 $\pm$ 1.352} \\
\midrule
QoE (acc)$\uparrow$ & 0.7152 & 0.555 $\pm$ 0.120 & 0.645 $\pm$ 0.122 & 0.510 $\pm$ 0.010 & 0.407 $\pm$ 0.190 & 0.670 $\pm$ 0.016 & 0.694 $\pm$ 0.018 & \textbf{0.722 $\pm$ 0.019}  \\ \hline
Covertype (acc)$\uparrow$ & \textbf{0.9581} & 0.816 $ \pm $ 0.003 & 0.834 $ \pm $ 0.003 & 0.846 $ \pm $ 0.073 & \textit{0.847 $ \pm $ 0.074} & 0.829 $\pm$ 0.001 & 0.799 $\pm$ 0.005 & 0.842 $\pm$ 0.075  \\ \hline
Higgs (acc)$\uparrow$ & \textbf{0.7478} & 0.667 $\pm$ 0.001 & 0.665 $\pm$ 0.000 & 0.499 $\pm$ 0.001 & 0.499 $\pm$ 0.001 & 0.675 $\pm$ 0.000 & 0.617 $\pm$ 0.000 & \textit{0.734 $\pm$ 0.008} \\ \hline
QoE (f1)$\uparrow$ & 0.6205 & 0.435 $\pm$ 0.155 & 0.512 $\pm$ 0.217 & 0.135 $\pm$ 0.002 & 0.112 $\pm$ 0.042 & 0.590 $\pm$ 0.012 & 0.577 $\pm$ 0.008 & \textbf{0.640 $\pm$ 0.029} \\ \hline
Covertype (f1)$\uparrow$ & \textbf{0.9373} & 0.630 $ \pm $ 0.023 & 0.666 $ \pm $ 0.015 & 0.682 $ \pm $ 0.010 & \textit{0.688 $ \pm $ 0.010} & 0.667 $ \pm $ 0.005 & 0.589 $\pm$ 0.005 & 0.672 $ \pm $ 0.013  \\ \hline
Higgs (f1)$\uparrow$ & \textbf{0.7477} & 0.667 $\pm$ 0.001 & 0.665 $\pm$ 0.000 & 0.499 $\pm$ 0.001 & 0.499 $\pm$ 0.001 & 0.675 $\pm$ 0.000 & 0.580 $\pm$ 0.000 & \textit{0.734 $\pm$ 0.008} \\
\bottomrule
\end{tabular}
\end{table*}

\begin{table}[!h]
\centering
\caption{Comparison of PPSO, OWQ, and Ours for average bit weights and weights with quantized activation.}
\label{tab:size comp MLP}
\small
\setlength{\tabcolsep}{2pt}
\begin{tabular}{c|c|cc|c|cc|c|cc}
\toprule
\multirow{2}{*}{Dataset} & \multicolumn{3}{c|}{PPSO} & \multicolumn{3}{c|}{OWQ} & \multicolumn{3}{c}{Ours} \\
\cline{2-10}
& w/o & w & act & w/o & w & act & w/o & w & act \\
\midrule
KVS & 1.29 & 2.04 & 8 & 4.28 & 4.28 & 8.24 & \textbf{1.01} & \textbf{1.339} & \textbf{4.855}   \\ \hline
VoD & 3.79 & 1.895 & 8 & 4.28 & 4.28 & 8.24 & \textbf{1.075} & \textbf{1.13} & \textbf{4.779}  \\ \hline
RSS & 3.29 & \textbf{1.29} & 8 & 4.28 & 4.28 & 8.24 & \textbf{1.186} & 3.139 & \textbf{6.678}    \\
\midrule
QoE & 3.64 & 2.5 & 8 & 4.28 & 4.28 & 8.24 & \textbf{2.419} & \textbf{1.001} & \textbf{4.005}   \\ \hline
Covertype & 1.895 & 4.25 & 8 & 4.28 & 4.28 & 8.24 & \textbf{1.394} & \textbf{1.55} & \textbf{4.685}  \\ \hline
Higgs & \textbf{2.04} & 2.64 & 8 & 4.28 & 4.28 & 8.24 & 2.32 & \textbf{2.435} & \textbf{4.115} \\
\bottomrule
\end{tabular}
\end{table}

\section{Theoretical Analysis}
\label{sec:theory}

We provide two formal guarantees for NMP-QAT; full proofs and 
propositions are in Appendix~\ref{app:proofs}.

\textbf{Convergence.} Under $L$-smoothness of the surrogate loss 
$\widetilde{\mathcal{L}}$ and bounded STE gradient variance, SGD on 
$\widetilde{\mathcal{L}}$ achieves the standard 
$\mathcal{O}(T^{-1/2})$ stationarity rate. The hard-quantized 
objective converges to an $\mathcal{O}(\delta^2)$ neighborhood of 
stationarity, where $\delta$ bounds the STE-induced gradient bias 
(Proposition~\ref{prop:convergence}).

\textbf{Compression-accuracy trade-off.} The loss gap between 
NMP-quantized and full-precision models is bounded by 
$\frac{L}{2}\rho(\bar{b}_{\mathrm{NMP}})^2$, where 
$\rho(\bar{b}_{\mathrm{NMP}})^2 = \sum_j d_{\mathrm{in}} 
\epsilon_j(b_j)$ aggregates per-neuron quantization errors 
$\epsilon_j(b_j) \propto \sigma_j^2 / 4^{b_j}$ 
(Proposition~\ref{prop:tradeoff}). Since $\epsilon_j$ decays 
exponentially with $b_j$ and scales with neuron weight variance 
$\sigma_j^2$, allocating higher precision to more sensitive neurons 
minimizes $\rho$ at a fixed average bit budget, motivating 
NMP-QAT's per-neuron precision design.

\begin{table*}[!h]
\centering
\caption{Performance comparison across different datasets with Tabformer model under different quantization settings.}
\scriptsize
\setlength{\tabcolsep}{3pt}
\renewcommand{\arraystretch}{1.2}
\label{tab:tabformer_comp}

\begin{tabular}{c|c|ccccc|c|c|c}
\toprule
\multirow{2}{*}{Dataset} & \multirow{2}{*}{FP} & \multicolumn{5}{c|}{QAT}  & \multirow{2}{*}{OWQ} & \multirow{2}{*}{PPSO} & \multirow{2}{*}{Ours} \\ \cline{3-7}
& & 1-bit & 1.58-bit & 4-bit & 8-bit & 16-bit &  \\
\midrule
KVS (MSE$\downarrow$) & \textbf{2.295} & 2.950$\pm$0.721 & 2.890$\pm$0.799 & 3.068$\pm$0.742 & 3.110$\pm$0.747 & 2.996$\pm$0.713 & \textit{2.410$\pm$0.015} & 3.098$\pm$0.479 & 2.616 $\pm$ 0.284  \\ \hline
VoD (MSE$\downarrow$) & 14.75 & 17.481$\pm$0.310 & 16.033$\pm$0.332 & 14.313$\pm$0.582 & 14.029$\pm$0.701 & 14.011$\pm$0.701 & 15.184$\pm$0.040 & 15.853$\pm$0.993 & \textbf{12.05 $\pm$ 0.97} \\  \hline
RSS (MSE$\downarrow$) & \textbf{8.262} & 15.847$\pm$0.065 & 12.834$\pm$0.022 & 9.374$\pm$0.043 & 8.746$\pm$0.036 & 8.643$\pm$0.031 & 9.225$\pm$0.031 & 18.847$\pm$1.448 & \textit{8.606 $\pm$ 0.064}  \\
\midrule
QoE (acc $\uparrow$) & 0.7512 & 0.749$\pm$0.037 & 0.755$\pm$0.023 & 0.777$\pm$0.005 & 0.778$\pm$0.001 & 0.779 $\pm$ 0.015 & 0.779$\pm$0.001 & 0.779$\pm$0.010 & \textbf{0.782} $\pm$ 0.002 \\ \hline
Covertype (acc $\uparrow$) & 0.9395 & 0.766$\pm$0.002 & 0.812$\pm$0.001 & 0.921$\pm$0.002 & 0.933$\pm$0.002 & 0.9413 $\pm$ 0.003 & 0.929$\pm$0.001 & 0.773$\pm$0.003 & \textbf{0.947 $\pm$ 0.005}  \\ \hline
Higgs (acc $\uparrow$) & 0.7561 & 0.692$\pm$0.001 & 0.732$\pm$0.001 & 0.754$\pm$0.000 & 0.755$\pm$0.001 & \textbf{0.763 $\pm$ 0.051}  & 0.753$\pm$0.000 & 0.695$\pm$0.014 & 0.7458 $\pm$ 0.015 \\ 
\bottomrule
\end{tabular}
\end{table*}

\begin{table*}[!h]
\centering
\caption{Performance comparison with TabFormer model under weights+activation quantization.}
\scriptsize
\setlength{\tabcolsep}{3pt}
\renewcommand{\arraystretch}{1.2}
\label{tab:tabformer_comp_act}

\begin{tabular}{c|c|ccccc|c|c|c}
\toprule
\multirow{2}{*}{Dataset} & \multirow{2}{*}{FP} & \multicolumn{5}{c|}{QAT}  & \multirow{2}{*}{OWQ} & \multirow{2}{*}{PPSO} & \multirow{2}{*}{Ours} \\ \cline{3-7}
& & 1-bit & 1.58-bit & 4-bit & 8-bit & 16-bit &  \\
\midrule
KVS (MSE$\downarrow$) & \textbf{2.295} & 426.7$\pm$0.7 & 426.7$\pm$0.7 & 426.7$\pm$0.7 & 426.7$\pm$0.7 & 426.7$\pm$0.7 & 312.4$\pm$0.0 & 426.7$\pm$0.7 & \textit{4.559$\pm$0.764}  \\ \hline
VoD (MSE$\downarrow$) & \textbf{14.75} & 18.752$\pm$0.283 & 17.977$\pm$0.361 & 18.357$\pm$0.690 & 17.503$\pm$0.491 & 17.992$\pm$0.388 & 18.295$\pm$0.245 & 17.764$\pm$0.255 & \textit{17.379$\pm$0.354} \\  \hline
RSS (MSE$\downarrow$) & \textbf{8.262} & 4335.0$\pm$8.1 & 4335.0$\pm$8.1 & 9442.0$\pm$11.5 & 9442.0$\pm$11.5 & 9442.0$\pm$11.5 & 9178.6$\pm$16.2 & 4335.0$\pm$8.1 & \textit{167.65$\pm$4.99}  \\
\midrule
QoE (acc $\uparrow$) & 0.7512 & 0.688$\pm$0.013 & 0.713$\pm$0.007 & 0.767$\pm$0.008 & 0.774$\pm$0.005 & 0.774$\pm$0.003 & \textbf{0.775$\pm$0.002} & 0.738$\pm$0.023 & 0.769$\pm$0.024 \\ \hline
Covertype (acc $\uparrow$) & \textbf{0.9395} & 0.723$\pm$0.003 & 0.748$\pm$0.003 & 0.798$\pm$0.003 & 0.802$\pm$0.002 & 0.81 $\pm$ 0.003 & \textit{0.812$\pm$0.001} & 0.733$\pm$0.010 & 0.793$\pm$0.008  \\ \hline
Higgs (acc $\uparrow$) & \textbf{0.7561} & 0.648$\pm$0.002 & 0.657$\pm$0.000 & 0.671$\pm$0.002 & 0.673$\pm$0.001 & \textit{0.678 $\pm$ 0.045}  & 0.673$\pm$0.001 & 0.672$\pm$0.000 & \textit{0.687$\pm$0.005} \\ 
\bottomrule
\end{tabular}
\end{table*}

\section{Experimental Details} \label{sec:exp_details}

We evaluate NMP-QAT against a full-precision baseline, two mixed-precision methods, OWQ~\cite{lee2024owq} and 
PPSO~\cite{fang2025mixed}, and uniform QAT configurations, on two architectures: a 4-layer MLP (512 neurons per layer) and TabFormer~\cite{padhi2021tabular}, a transformer-based tabular foundation model. Experiments cover six datasets: four telecom workloads (KVS, VoD, QoE, RSS) -- out of which KVS, VoD, RSS are regression tasks and QoE is a classification task, representing realistic 6G scenarios, and two standard classification benchmarks (CoverType, Higgs) to validate generality. Further dataset details are given in Appendix~\ref{app_datasets}.

\subsection{Results and Discussions} \label{sec:results}

Models are trained for 100 epochs with early stopping (patience=20) using SGD (lr=$1e-3$), repeated three times; we report mean $\pm$ std. Bold marks the best overall result; italics indicate the best quantized result when full-precision leads. Weights-only quantization results are provided in Appendix~\ref{app:weights_only}, where NMP-QAT consistently achieves the best or near-best performance across all datasets. Here we focus on the more challenging weights+activations setting.

While uniform QAT, PPSO, and OWQ fix activations at 8-bit, NMP-QAT initializes activations at 4-bit and expands precision progressively (Section~\ref{sec:neuron-act}). This asymmetry with the 1-bit weight initialization is deliberate: unlike weights, which are static after training and tolerate aggressive low-bit initialization, activations are dynamic quantities that vary per input and exhibit larger dynamic ranges. Initializing activations at 1-bit causes representational collapse early in training, destabilizing gradient flow before the
precision expansion mechanism can compensate. A 4-bit starting point provides sufficient numerical range for stable training while preserving meaningful headroom for upward adaptation. 

As shown in Table~\ref{tab:main_results_act}, NMP-QAT achieves the best quantized results on KVS, VoD, QoE, and Higgs, with 8-bit QAT marginally better only on CoverType. Notably, binary and ternary QAT occasionally outperform higher-bit schemes due to regularization effects that reduce overfitting, with ternary quantization further benefiting from zero-induced sparsity; beyond 4-bit, quantization error is already small and additional bits may hurt generalization. Table~\ref{tab:size comp MLP} shows that NMP-QAT uses at most $2.42$ bits for weights (mostly $1$--$2$ bits), versus $3$--$4$ bits for PPSO and a fixed $4.28$ bits for OWQ. For activations, NMP-QAT requires only $4$--$5$ bits across most datasets, rising to $6.67$ bits only for RSS, well below the 8-bit baseline used by all competing methods. 

For TabFormer (Table~\ref{tab:tabformer_comp}), we use the same settings as before. Due to QKV pathway interactions, extremely low-bit precision performs poorly under uniform QAT. NMP-QAT achieves the best performance on VoD, QoE, and CoverType, with only $0.05\%-1\%$ degradation versus full precision, demonstrating adaptive precision allocation in complex attention-based architectures. Table~\ref{tab:tabformer_comp_act} evaluates TabFormer under joint weights+activations quantization. Aggressive activation quantization severely degrades all uniform QAT baselines, particularly for regression (KVS, RSS). This highlights the inherent sensitivity of attention layers to activation bit-width. NMP-QAT remains substantially more stable: it reduces MSE by over two orders of magnitude on KVS and RSS relative to uniform QAT, and surpasses even the 16-bit baseline on Higgs. This demonstrates that neuron-level adaptive activation precision is critical for attention-heavy architectures.


Across all datasets and architectures, NMP-QAT delivers the strongest gains on telecom datasets spanning regression (KVS, VoD, RSS) and classification (QoE) tasks, where heterogeneous feature distributions and dynamic activation ranges make quantization especially challenging. On non-telecom classification benchmarks (CoverType, Higgs), where distributions are more stable, NMP-QAT still matches or exceeds baselines at substantially lower bit-widths. Together, these results confirm that NMP-QAT is both effective for demanding 6G edge deployments and general enough to perform robustly on standard benchmarks.

\subsection{Limitations}
\label{sec:limitations}

While NMP-QAT demonstrates strong performance across diverse datasets
and architectures, we highlight the following limitations.

\begin{itemize}
    \item \textit{Hardware realization.} Sub 8-bit arithmetic is not universally supported; realizing theoretical savings requires mixed-precision capable accelerators.
    \item \textit{Network depth.} Experiments use a 4-layer MLP; behavior across shallower and deeper architectures remains understudied.
    \item \textit{Hardware-aware evaluation.} Latency and energy measurements on real deployment targets are left for future work.
\end{itemize}

\section{Conclusion and Future Work}

We introduced NMP-QAT, a framework that adaptively learns discrete precision at the neuron level, progressively increasing bit-width only where needed during training. Unlike static layer- or channel-level mixed-precision methods, NMP-QAT achieves aggressive compression without sacrificing accuracy. Experiments on six datasets show consistently superior accuracy-memory trade-offs over OWQ and PPSO, highlighting its potential for 6G edge deployments under strict energy and latency constraints.

Several directions remain for future work. Aggressive activation quantization remains challenging in attention-based tabular models, motivating improved scaling strategies or hybrid schemes for sensitive layers. Since most neurons converge to 1-bit (Fig.~\ref{fig:Neuron_level_MP_MLP}), many may contribute minimally and could be pruned, suggesting that combining precision learning with structured sparsification could yield further efficiency gains. Finally, we plan to investigate how architectural components such as dropout and skip connections influence extreme quantization and generalization under ultra-low-bit regimes.

\bibliography{example_paper}
\bibliographystyle{icml2026}

\section*{Impact Statement}

This paper presents NMP-QAT, a neuron-level mixed-precision 
quantization-aware training framework aimed at advancing efficient 
machine learning for resource-constrained deployments. The primary 
motivation is to enable accurate, low-footprint inference on edge 
devices, with particular relevance to 6G network intelligence, where 
energy efficiency, low latency, and memory constraints are critical.

By reducing the memory and computational requirements of neural 
network inference, this work contributes to more sustainable AI 
deployment, lowering energy consumption in large-scale network 
infrastructure. At the same time, enabling more capable models on 
edge devices raises general considerations around data privacy and 
autonomous decision-making in communication systems, though these 
are not unique to our method and are broadly shared across the field 
of efficient deep learning.

We do not foresee specific harmful applications arising directly from 
this work. The datasets used are either publicly available benchmarks 
or synthetically generated, and no personally identifiable information 
is involved.

\section*{Use of Generative AI}

The authors acknowledge the use of generative AI tools, specifically Microsoft Copilot and Claude, to assist with grammar checking in order to improve the clarity and readability of the manuscript. 

\section*{Acknowledgement}

This work was partially supported by the Wallenberg Al, Autonomous Systems and Software Program (WASP) funded by the Knut and Alice Wallenberg Foundation.

\section*{3MT presentation}

The 3 minute presentation for this paper is available at https://youtu.be/jiBRLFmm3Y8.

\newpage
\appendix
\onecolumn

\section{Related Work} \label{sec:related_works}

Mixed-precision quantization has emerged as a practical alternative to 
uniform-precision schemes, motivated by heterogeneous quantization 
sensitivity across network components~\cite{rakka2024review}. Early 
methods refine allocation granularity beyond layers: channel-level 
schemes assign bits per channel or group to better capture within-layer 
variability~\cite{chen2024channel}, while outlier-aware approaches 
(OWQ~\cite{lee2024owq}, AWQ~\cite{lin2024awq}, 
SmoothQuant~\cite{xiao2023smoothquant}) mitigate distortion from 
high-magnitude weights and activations, achieving strong accuracy under 
aggressive bit budgets. These methods, however, rely on post-training 
heuristics and operate at coarse granularity, producing static 
precision assignments that cannot adapt during training.

Search-based methods cast bit allocation as an optimization problem, 
using evolution-guided search~\cite{dong2023emq} or constrained 
particle-swarm algorithms~\cite{fang2025mixed} to explore precision 
configurations under bit-budget constraints. While more principled than 
heuristics, these procedures are computationally expensive and still 
yield \emph{static} assignments, limiting adaptability and scalability 
for 6G-scale systems.

Finer-grained, training-integrated approaches remain underexplored. 
Q-BERT~\cite{shen2020q} demonstrates that quantization sensitivity 
varies significantly across individual neurons, motivating neuron-level 
control. The Neuron-by-Neuron QAT method~\cite{sher2023neuron} 
progressively quantizes neurons via freezing masks and straight-through 
estimators, outperforming layer- and channel-wise schemes in low-bit 
regimes. However, it targets \emph{uniform} low-bit quantization 
(2-8 bits) through multi-stage freezing procedures, making extension 
to mixed-precision settings non-trivial. Neuron-level 
\emph{mixed}-precision methods, particularly those jointly quantizing 
weights and activations, remain largely unexplored. NMP-QAT addresses 
this gap directly: it learns discrete per-neuron precision during 
training, starting low and expanding only when gradient signals demand 
it, with no search overhead or multi-stage pipelines.

\section{Theoretical Analysis}
\label{app:proofs}

\subsection{Assumptions}

\begin{assumption}[Smooth surrogate objective]
\label{ass:smooth}
The surrogate objective $\widetilde{\mathcal{L}}(\theta)$ is
$L$-smooth:
\[
\|\nabla\widetilde{\mathcal{L}}(\theta)
-
\nabla\widetilde{\mathcal{L}}(\theta')\|
\le
L\|\theta-\theta'\|
\quad
\text{for all } \theta,\theta' .
\]
\end{assumption}
The use of an STE-induced surrogate or coarse-gradient direction is
motivated by prior analyses of STE-based quantized neural network
training, which treat the STE update as a biased first-order oracle rather
than the true gradient of the discontinuous quantized objective
\cite{long2021learning}.

\begin{assumption}[Bounded STE bias and stochastic variance]
\label{ass:ste}
Let $\mathcal{L}(\theta)$ denote the true hard-quantized objective, and
let $\widetilde{\mathcal{L}}(\theta)$ denote the differentiable surrogate
objective induced by the straight-through estimator. We assume that the
surrogate-gradient bias relative to the true gradient is uniformly
bounded along the iterates:
\[
\bigl\|
\nabla_{\mathbf W}\widetilde{\mathcal L}(\theta_t)
-
\nabla_{\mathbf W}\mathcal L(\theta_t)
\bigr\|
\le \delta ,
\]
for some constant $\delta\ge 0$ and for all $t$.

The stochastic STE gradient is written as
\[
\hat g_t
=
\nabla \widetilde{\mathcal L}(\theta_t)+\xi_t,
\]
where $\xi_t$ denotes the mini-batch noise. We assume
\[
\mathbb E[\xi_t\mid \theta_t]=0,
\qquad
\mathbb E[\|\xi_t\|^2\mid \theta_t]\le \sigma^2 .
\]
Equivalently,
\[
\mathbb E[\hat g_t\mid \theta_t]
=
\nabla \widetilde{\mathcal L}(\theta_t),
\qquad
\mathbb E[
\|\hat g_t-\nabla \widetilde{\mathcal L}(\theta_t)\|^2
\mid \theta_t]
\le \sigma^2 .
\]
\end{assumption}

\subsection{Convergence of Joint Weight-Precision Optimization}

Let $\theta = (\mathbf{W}, \{s_j\})$, let
$\widetilde{\mathcal{L}}(\theta)$ denote the differentiable surrogate
loss used in the backward pass, and let $\mathcal{L}(\theta)$ denote the
hard-quantized loss induced by the forward pass. We assume that
$\widetilde{\mathcal{L}}$ is $L$-smooth and that the stochastic STE
gradient is an unbiased estimator of
$\nabla \widetilde{\mathcal{L}}(\theta_t)$ with bounded mini-batch
variance $\sigma^2$. We further assume that the STE-induced surrogate
gradient differs from the true discrete gradient by at most $\delta$ along
the iterates.

\begin{proposition}[Convergence rate]
\label{prop:convergence}
Under Assumptions~\ref{ass:smooth}--\ref{ass:ste}, SGD on
$\widetilde{\mathcal{L}}$ with step size $\eta = 1/\sqrt{T}$ satisfies,
for $T\ge L^2$,
\begin{equation}
\frac{1}{T}\sum_{t=0}^{T-1}
\mathbb{E}\!\left[
\|\nabla\widetilde{\mathcal{L}}(\theta_t)\|^2
\right]
\le
\frac{
2\bigl(\widetilde{\mathcal{L}}(\theta_0)
-\widetilde{\mathcal{L}}^*\bigr)
+L\sigma^2
}{\sqrt{T}} .
\end{equation}
Thus, SGD achieves the standard $\mathcal{O}(T^{-1/2})$ stationarity rate
with respect to the surrogate objective. Moreover, the corresponding
stationarity gap for the true hard-quantized objective satisfies
\begin{equation}
\frac{1}{T}\sum_{t=0}^{T-1}
\mathbb{E}\!\left[
\|\nabla_{\mathbf W}\mathcal{L}(\theta_t)\|^2
\right]
\le
\frac{
4\bigl(\widetilde{\mathcal{L}}(\theta_0)
-\widetilde{\mathcal{L}}^*\bigr)
+2L\sigma^2
}{\sqrt{T}}
+2\delta^2 .
\end{equation}
Therefore, the method converges to an $\mathcal{O}(T^{-1/2})$ stationary
point of the surrogate objective and to an $\mathcal{O}(\delta^2)$
neighborhood of stationarity for the hard-quantized objective.
\end{proposition}

\begin{proof}
By Assumption~\ref{ass:smooth}, $\widetilde{\mathcal L}$ is
$L$-smooth. Using the SGD update
\[
\theta_{t+1}=\theta_t-\eta \hat g_t,
\]
the descent lemma gives
\[
\widetilde{\mathcal L}(\theta_{t+1})
\le
\widetilde{\mathcal L}(\theta_t)
-\eta
\left\langle
\nabla \widetilde{\mathcal L}(\theta_t),\hat g_t
\right\rangle
+\frac{L\eta^2}{2}\|\hat g_t\|^2 .
\]
By Assumption~\ref{ass:ste},
\[
\hat g_t
=
\nabla \widetilde{\mathcal L}(\theta_t)+\xi_t,
\qquad
\mathbb E[\xi_t\mid \theta_t]=0,
\qquad
\mathbb E[\|\xi_t\|^2\mid \theta_t]\le \sigma^2 .
\]
Therefore,
\[
\mathbb E[\hat g_t\mid \theta_t]
=
\nabla \widetilde{\mathcal L}(\theta_t)
\]
and
\[
\mathbb E[\|\hat g_t\|^2\mid \theta_t]
=
\|\nabla \widetilde{\mathcal L}(\theta_t)\|^2
+
\mathbb E[\|\xi_t\|^2\mid \theta_t]
\le
\|\nabla \widetilde{\mathcal L}(\theta_t)\|^2+\sigma^2 .
\]
Taking conditional expectation in the descent inequality yields
\[
\mathbb E[
\widetilde{\mathcal L}(\theta_{t+1})
\mid \theta_t]
\le
\widetilde{\mathcal L}(\theta_t)
-
\eta\|\nabla \widetilde{\mathcal L}(\theta_t)\|^2
+
\frac{L\eta^2}{2}
\left(
\|\nabla \widetilde{\mathcal L}(\theta_t)\|^2+\sigma^2
\right).
\]
If $\eta\le 1/L$, then
\[
-\eta+\frac{L\eta^2}{2}
\le
-\frac{\eta}{2},
\]
and hence
\[
\mathbb E[
\widetilde{\mathcal L}(\theta_{t+1})
\mid \theta_t]
\le
\widetilde{\mathcal L}(\theta_t)
-
\frac{\eta}{2}
\|\nabla \widetilde{\mathcal L}(\theta_t)\|^2
+
\frac{L\eta^2\sigma^2}{2}.
\]
Taking total expectation and summing from $t=0$ to $T-1$ gives
\[
\frac{\eta}{2}
\sum_{t=0}^{T-1}
\mathbb E
\|\nabla \widetilde{\mathcal L}(\theta_t)\|^2
\le
\widetilde{\mathcal L}(\theta_0)
-
\widetilde{\mathcal L}^*
+
\frac{L\eta^2\sigma^2T}{2}.
\]
Dividing by $\eta T/2$ yields
\[
\frac{1}{T}
\sum_{t=0}^{T-1}
\mathbb E
\|\nabla \widetilde{\mathcal L}(\theta_t)\|^2
\le
\frac{
2\bigl(
\widetilde{\mathcal L}(\theta_0)
-
\widetilde{\mathcal L}^*
\bigr)
}{\eta T}
+
L\eta\sigma^2 .
\]
Choosing $\eta=1/\sqrt{T}$ and assuming $T\ge L^2$ so that
$\eta\le 1/L$, we obtain
\[
\frac{1}{T}
\sum_{t=0}^{T-1}
\mathbb E
\|\nabla \widetilde{\mathcal L}(\theta_t)\|^2
\le
\frac{
2\bigl(
\widetilde{\mathcal L}(\theta_0)
-
\widetilde{\mathcal L}^*
\bigr)
+
L\sigma^2
}{\sqrt T}.
\]
This proves the $\mathcal O(T^{-1/2})$ convergence rate for stationarity
of the surrogate objective.

It remains to relate surrogate stationarity to stationarity of the
hard-quantized objective. By Assumption~\ref{ass:ste},
\[
\left\|
\nabla_{\mathbf W}\mathcal L(\theta_t)
-
\nabla_{\mathbf W}\widetilde{\mathcal L}(\theta_t)
\right\|
\le \delta .
\]
Thus, by the inequality $\|a+b\|^2\le 2\|a\|^2+2\|b\|^2$,
\[
\|\nabla_{\mathbf W}\mathcal L(\theta_t)\|^2
\le
2\|\nabla_{\mathbf W}\widetilde{\mathcal L}(\theta_t)\|^2
+
2\delta^2
\le
2\|\nabla\widetilde{\mathcal L}(\theta_t)\|^2
+
2\delta^2 .
\]
Averaging over $t=0,\ldots,T-1$ and using the surrogate bound gives
\[
\frac{1}{T}
\sum_{t=0}^{T-1}
\mathbb E
\|\nabla_{\mathbf W}\mathcal L(\theta_t)\|^2
\le
\frac{
4\bigl(
\widetilde{\mathcal L}(\theta_0)
-
\widetilde{\mathcal L}^*
\bigr)
+
2L\sigma^2
}{\sqrt T}
+
2\delta^2 .
\]
Therefore, the method converges to an $\mathcal O(T^{-1/2})$ stationary
point of the surrogate objective and to an $\mathcal O(\delta^2)$
neighborhood of stationarity for the hard-quantized objective.

Finally, the precision strengths $\{s_j\}$ are part of the joint
parameter vector $\theta=(\mathbf W,\{s_j\})$. Since the surrogate
objective is differentiable with respect to all components of $\theta$,
the same descent argument applies to joint weight-precision optimization.
\end{proof}

\subsection{Compression-Accuracy Trade-off}

Let $\mathbf W^*$ denote the full-precision optimum and let
$\mathbf W^{(q)}$ be its NMP-quantized counterpart. For neuron $j$, let $\epsilon_j(b_j) = {\sigma_j^2}/{3\cdot 4^{b_j}}$  denote the mean-squared quantization error of a $b_j$-bit uniform quantizer, where $\sigma_j^2$ is the variance of the weights associated
with neuron $j$~\cite{gersho2012vector}. Let $d_{\mathrm{in}}$ be the
number of input weights per neuron. The total quantization perturbation
is bounded by
\begin{equation}
\|\mathbf{W}^{(q)}-\mathbf{W}^*\|_F^2
\;\leq\;
\sum_{j=1}^{N} d_{\mathrm{in}}\epsilon_j(b_j)
\;=:\;
\rho\!\left(\bar b_{\mathrm{NMP}}\right)^2 ,
\end{equation}
where $\bar b_{\mathrm{NMP}} = \frac{1}{N}\sum_{j=1}^{N} b_j$. Since $\epsilon_j(b_j)$ decreases exponentially in $b_j$ and scales
linearly with $\sigma_j^2$, an optimal allocation under a fixed average
bit budget assigns relatively higher precision to neurons with larger
weight variance. This directly motivates the per-neuron precision design
of NMP-QAT.

\begin{proposition}[Compression-accuracy gap]
\label{prop:tradeoff}
Assume $\mathcal{L}$ is $L$-smooth in a neighborhood of the
full-precision optimum $\mathbf{W}^*$. Then the loss gap induced by
NMP quantization satisfies
\begin{equation}
\mathcal{L}^*_{\mathrm{NMP}}
-
\mathcal{L}^*_{\mathrm{fp}}
\;\leq\;
\frac{L}{2}
\rho\!\left(\bar b_{\mathrm{NMP}}\right)^2 .
\end{equation}
\end{proposition}

\begin{proof}
Let $\mathbf W^*$ be a full-precision minimizer of $\mathcal L$, and let
$\mathbf W^{(q)}$ be its NMP-quantized counterpart. By $L$-smoothness,
\[
\mathcal{L}(\mathbf{W}^{(q)})
\leq
\mathcal{L}(\mathbf{W}^*)
+
\left\langle
\nabla\mathcal{L}(\mathbf{W}^*),
\mathbf{W}^{(q)}-\mathbf{W}^*
\right\rangle
+
\frac{L}{2}
\|\mathbf{W}^{(q)}-\mathbf{W}^*\|_F^2 .
\]
Since $\mathbf W^*$ is a minimizer in the smooth region,
\[
\nabla\mathcal{L}(\mathbf W^*)=0.
\]
Therefore,
\[
\mathcal{L}(\mathbf{W}^{(q)})
-
\mathcal{L}(\mathbf{W}^*)
\leq
\frac{L}{2}
\|\mathbf{W}^{(q)}-\mathbf{W}^*\|_F^2 .
\]
By the definition of $\rho(\bar b_{\mathrm{NMP}})$,
\[
\|\mathbf{W}^{(q)}-\mathbf{W}^*\|_F^2
\leq
\rho\!\left(\bar b_{\mathrm{NMP}}\right)^2 .
\]
Hence,
\[
\mathcal{L}(\mathbf{W}^{(q)})
-
\mathcal{L}^*_{\mathrm{fp}}
\leq
\frac{L}{2}
\rho\!\left(\bar b_{\mathrm{NMP}}\right)^2 .
\]
Finally, since $\mathcal{L}^*_{\mathrm{NMP}}$ is the best achievable loss
under the NMP quantization constraint, we have
\[
\mathcal{L}^*_{\mathrm{NMP}}
\leq
\mathcal{L}(\mathbf{W}^{(q)}).
\]
Thus,
\[
\mathcal{L}^*_{\mathrm{NMP}}
-
\mathcal{L}^*_{\mathrm{fp}}
\leq
\frac{L}{2}
\rho\!\left(\bar b_{\mathrm{NMP}}\right)^2 .
\]
\end{proof}

\begin{corollary}
\label{cor:bitwidth}
For accuracy tolerance $\varepsilon>0$, it is sufficient that
\[
\rho\!\left(\bar b_{\mathrm{NMP}}\right)
\le
\sqrt{\frac{2\varepsilon}{L}} .
\]
Equivalently, the average bit-width should satisfy
\begin{equation}
\bar b_{\mathrm{NMP}}
\ge
\rho^{-1}\!\left(
\sqrt{\frac{2\varepsilon}{L}}
\right).
\end{equation}
Neuron-level precision allocation reduces
$\rho(\bar b_{\mathrm{NMP}})$ at a fixed average bit budget by assigning
higher precision to neurons with larger quantization sensitivity, thereby
yielding a smaller upper bound on the compression-induced loss gap than
uniform bit assignment.
\end{corollary}

\begin{proof}
From Proposition~\ref{prop:tradeoff}, we have
\[
\mathcal{L}^*_{\mathrm{NMP}}
-
\mathcal{L}^*_{\mathrm{fp}}
\leq
\frac{L}{2}
\rho\!\left(\bar b_{\mathrm{NMP}}\right)^2 .
\]
To guarantee
\[
\mathcal{L}^*_{\mathrm{NMP}}
-
\mathcal{L}^*_{\mathrm{fp}}
\leq
\varepsilon ,
\]
it is sufficient that
\[
\frac{L}{2}
\rho\!\left(\bar b_{\mathrm{NMP}}\right)^2
\leq
\varepsilon .
\]
Equivalently,
\[
\rho\!\left(\bar b_{\mathrm{NMP}}\right)
\leq
\sqrt{\frac{2\varepsilon}{L}} .
\]
Assuming $\rho$ is monotonically decreasing in the average bit budget,
this yields the sufficient condition
\[
\bar b_{\mathrm{NMP}}
\geq
\rho^{-1}\!\left(
\sqrt{\frac{2\varepsilon}{L}}
\right).
\]

Finally, because
\[
\epsilon_j(b_j)
=
\frac{\sigma_j^2}{3\cdot 4^{b_j}},
\]
the perturbation contribution of neuron $j$ decreases exponentially with
$b_j$ and increases linearly with $\sigma_j^2$. Therefore, under a fixed
average bit budget, assigning relatively more bits to neurons with larger
$\sigma_j^2$ reduces the total perturbation
$\rho(\bar b_{\mathrm{NMP}})$ compared with an allocation that ignores
neuron-level sensitivity. This gives a smaller upper bound on the
compression-induced loss gap.
\end{proof}

\section{Datasets} \label{app_datasets}

\textit{Received signal strength (RSS) prediction} dataset is Reference Signal Received Power (RSRP) or path loss dataset generated by the Ericsson internal simulator with 50k user equipments (UEs) across multiple frequency bands using measurements from a single primary carrier. RSS is a multiclass regression problem from a network with 9 cells, each having one primary and three secondary carriers, provides 9 input features (one per primary carrier) and 27 prediction targets (three per cell). 

\textit{Key-value-store (KVS)  and Video-on-demand (VoD) \cite{yanggratoke2015predicting} performance prediction} are two separate datasets that contain data collected by running KVS and VoD services on data data centre, respectively. These services are used by a device that tracks their service level metrics - in this case frame rate of the video being streamed and write response time of the KV-store. The prediction task is to predict the relevant service-level metric by utilizing metrics collected from the data center. Metrics such as the utilization of CPU, memory, network and disk.

\textit{Quality of Experience (QoE) prediction} involves predicting a user's satisfaction with a video watching experience by using device-level metrics such as the videos frame rate, buffering time and audio loss. 

\textit{Covertype and Higgs} datasets from non-telecom domains that are frequently used as benchmarks for tabular machine learning tasks. CoverType is a classification task aiming to predict forest cover type. Whereas Higgs is binary classification task of prediction whether a particle from high-energy physics experiments is a Higgs boson.

\section{Additional Results} \label{app:weights_only}

This appendix provides supplementary results that complement the
main paper. Specifically, we include: (i) the per-layer
mixed-precision precision distributions for weights and activations
under NMP-QAT (Fig.~\ref{fig:Neuron_level_MP_MLP}); (ii) MLP
performance under weights-only quantization
(Table~\ref{tab:main_results}); (iii) memory-utility trade-off
plots for both quantization settings
(Fig.~\ref{fig:theoretical_memory_footprint_MLP}); (iv) the effect
of model depth on NMP-QAT performance
(Fig.~\ref{fig:layer_impact_NMP_QAT}); and (v) wall-clock runtimes
for mixed-precision methods on TabFormer
(Table~\ref{tab:comp_runtime}).

Fig.~\ref{fig:Neuron_level_MP_MLP} shows the fraction of neurons (stacked bars) assigned to each bit-width across layers $L_1$-$L_4$, separately for weights (left) and activations (right). The majority of weights across all layers and datasets remain at 1-bit, confirming the aggressive compression achieved by NMP-QAT while selectively raising precision only where training demands it. An exception is Higgs, where the first layer ($L_1$) requires higher weight precision due to the larger input feature space, slightly raising the overall average bit-width reported in Table~\ref{tab:size comp MLP}. For activations, precision is consistently higher than for weights, reflecting the dynamic nature of activation distributions discussed in Section~\ref{sec:results}, and rises most noticeably for RSS, which exhibits the largest activation dynamic range among the evaluated datasets.

Table~\ref{tab:main_results} shows the results under weights-only quantization. NMP-QAT achieves the best
performance on all regression tasks (KVS, VoD, RSS) and on the classification tasks QoE and Higgs, surpassing all baselines including the full-precision model on the three regression datasets. This is consistent with prior observations that quantization can act as a regularizer, reducing overfitting on noisy telecom workloads. On CoverType, full precision leads and NMP-QAT achieves the best quantized result ($0.933$ accuracy, $0.897$ F1), marginally outperforming PPSO ($0.930$, $0.889$) at substantially lower bit-widths. Binary and ternary QAT occasionally match or exceed higher-bit schemes, particularly on regression tasks, due to regularization effects; ternary quantization additionally benefits from zero-induced sparsity. Beyond 4-bit, quantization error is already small and additional bits may hurt generalization.

Each point in Fig.~\ref{fig:theoretical_memory_footprint_MLP}) represents one method under weights-only (Fig.~\ref{fig:theoretical_memory_footprint_MLP}a, Fig.~\ref{fig:theoretical_memory_footprint_MLP}c) and weights+activations (Fig.~\ref{fig:theoretical_memory_footprint_MLP}b, Fig.~\ref{fig:theoretical_memory_footprint_MLP}d)
quantization, for classification and regression tasks respectively.
NMP-QAT (marked by a square) consistently occupies the
low-memory/high-utility region across all datasets and both
quantization settings, achieving the best or near-best utility at
the lowest theoretical memory footprint. Competing methods such as
PPSO and OWQ either require higher memory for comparable utility or
sacrifice utility at similar compression levels.

Fig.~\ref{fig:layer_impact_NMP_QAT}) shows the effect of model depth. Accuracy and F1-score rise progressively with the number of layers for classification datasets (Fig.~\ref{fig:layer_impact_NMP_QAT}a), and normalized MSE decreases for regression datasets (Fig.~\ref{fig:layer_impact_NMP_QAT}b), confirming that NMP-QAT effectively leverages additional model capacity even under quantization constraints. This trend holds across both telecom (KVS, VoD, RSS, QoE) and non-telecom (CoverType, Higgs) datasets, suggesting that the neuron-level precision expansion mechanism scales well with model depth.

Table~\ref{tab:comp_runtime} shows the comparison of wall clock runtime. OWQ is the fastest method overall, as its post-training optimization requires no iterative forward-backward passes at scale. PPSO is the slowest by a wide margin due to its particle swarm search, requiring up to $3840$ minutes on Higgs. NMP-QAT achieves an approx $18.5\times$-$75\times$ runtime reduction over PPSO across all datasets while delivering superior accuracy and compression, confirming its practical viability for time-sensitive and resource-constrained deployments.

\begin{figure}[t]
\centering
\begin{subfigure}[t]{0.33\linewidth}
    \centering
    \includegraphics[width=\linewidth]{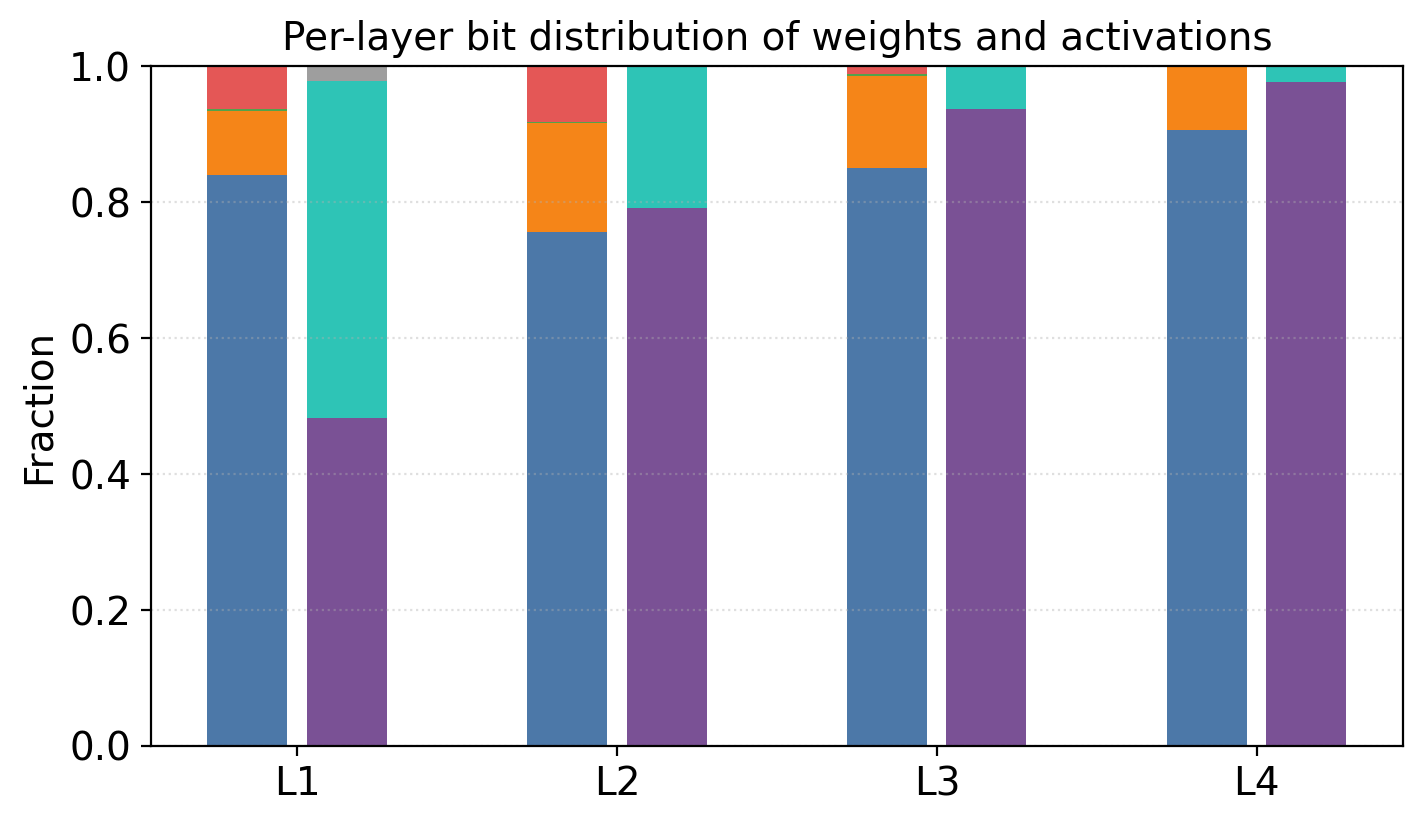}
    \caption{KVS}
\end{subfigure}
\hfill
\begin{subfigure}[t]{0.33\linewidth}
    \centering
    \includegraphics[width=\linewidth]{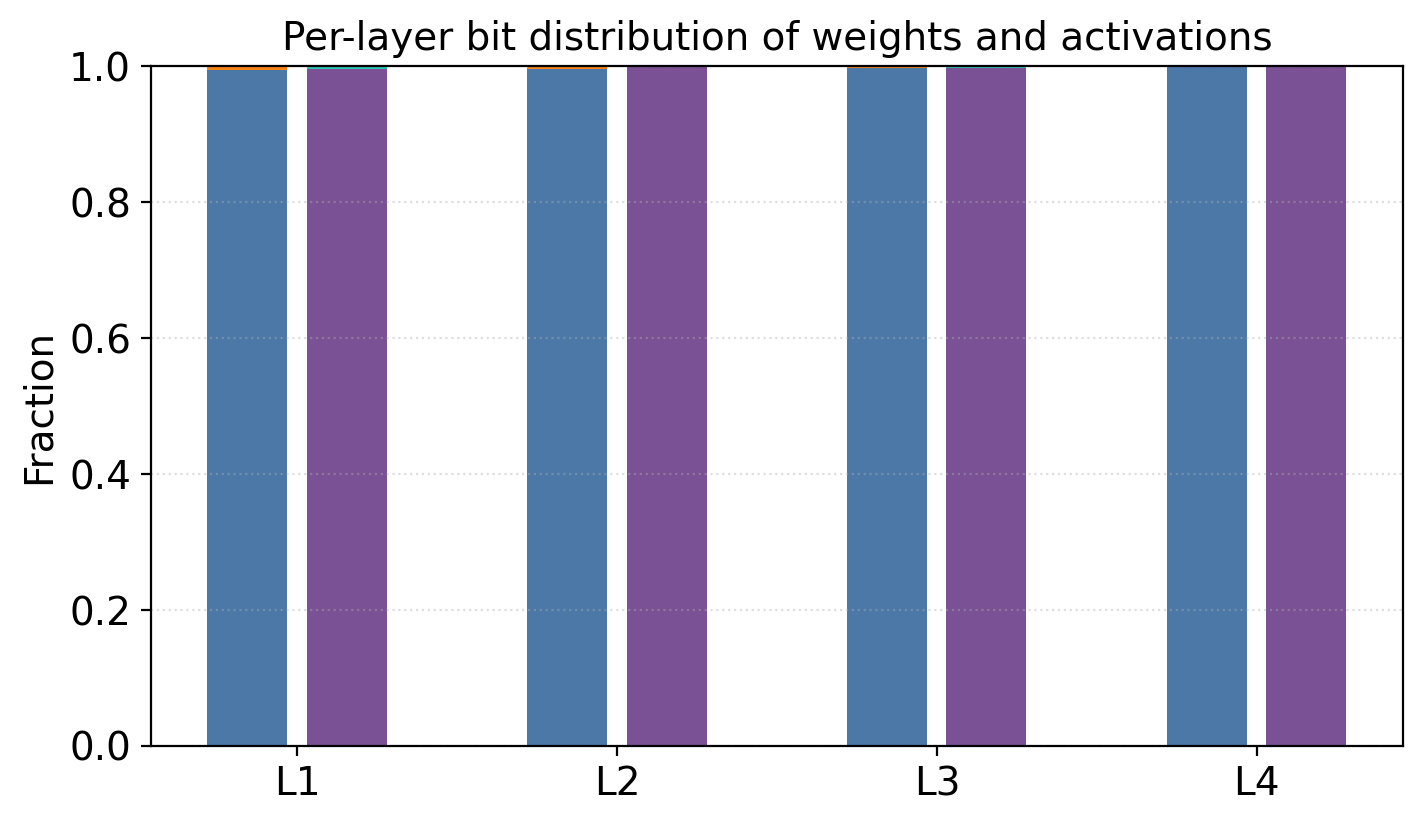}
    \caption{QoE}
\end{subfigure}
\hfill
\begin{subfigure}[t]{0.33\linewidth}
    \centering
    \includegraphics[width=\linewidth]{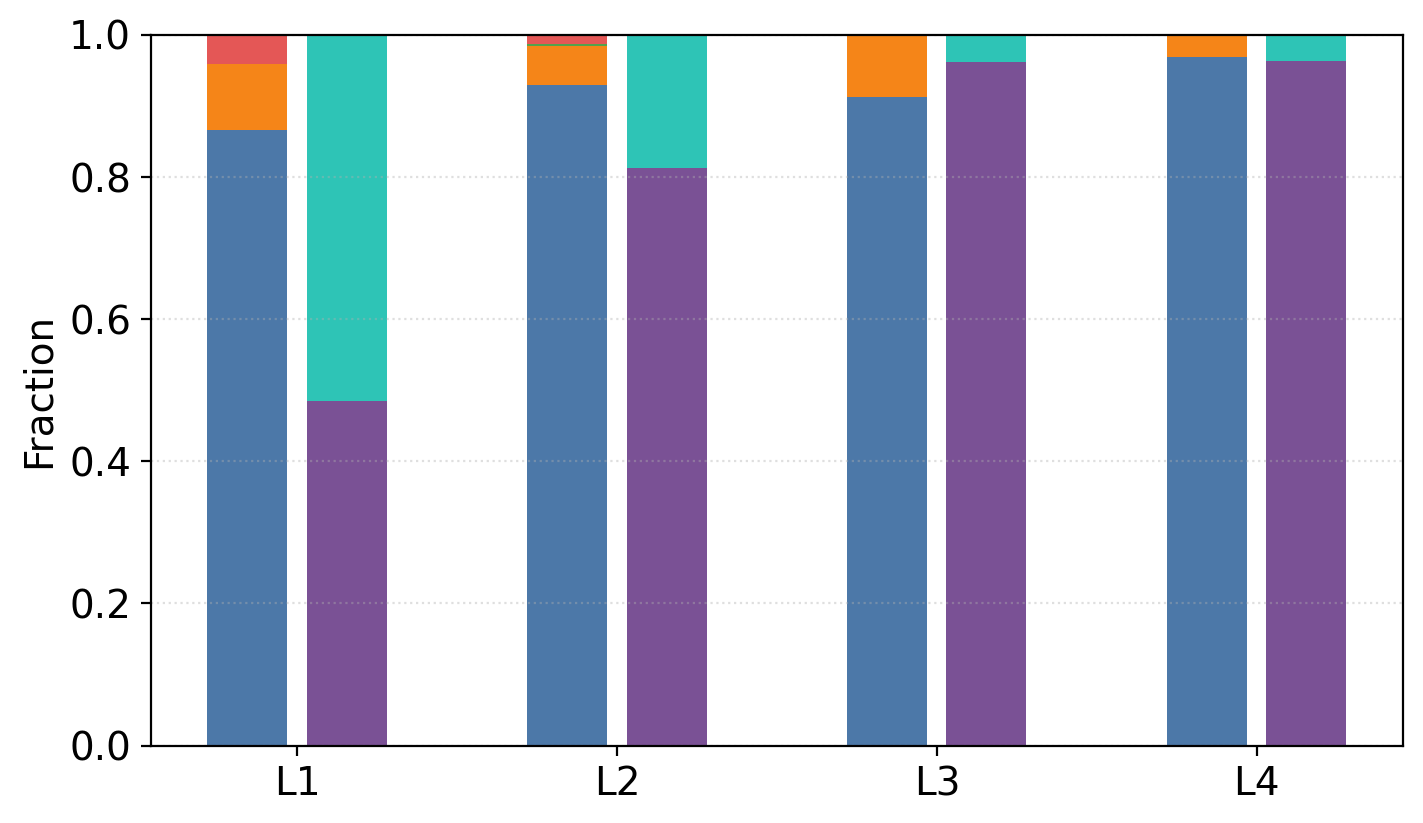}
    \caption{VoD}
\end{subfigure}

\begin{subfigure}[t]{0.33\linewidth}
    \centering
    \includegraphics[width=\linewidth]{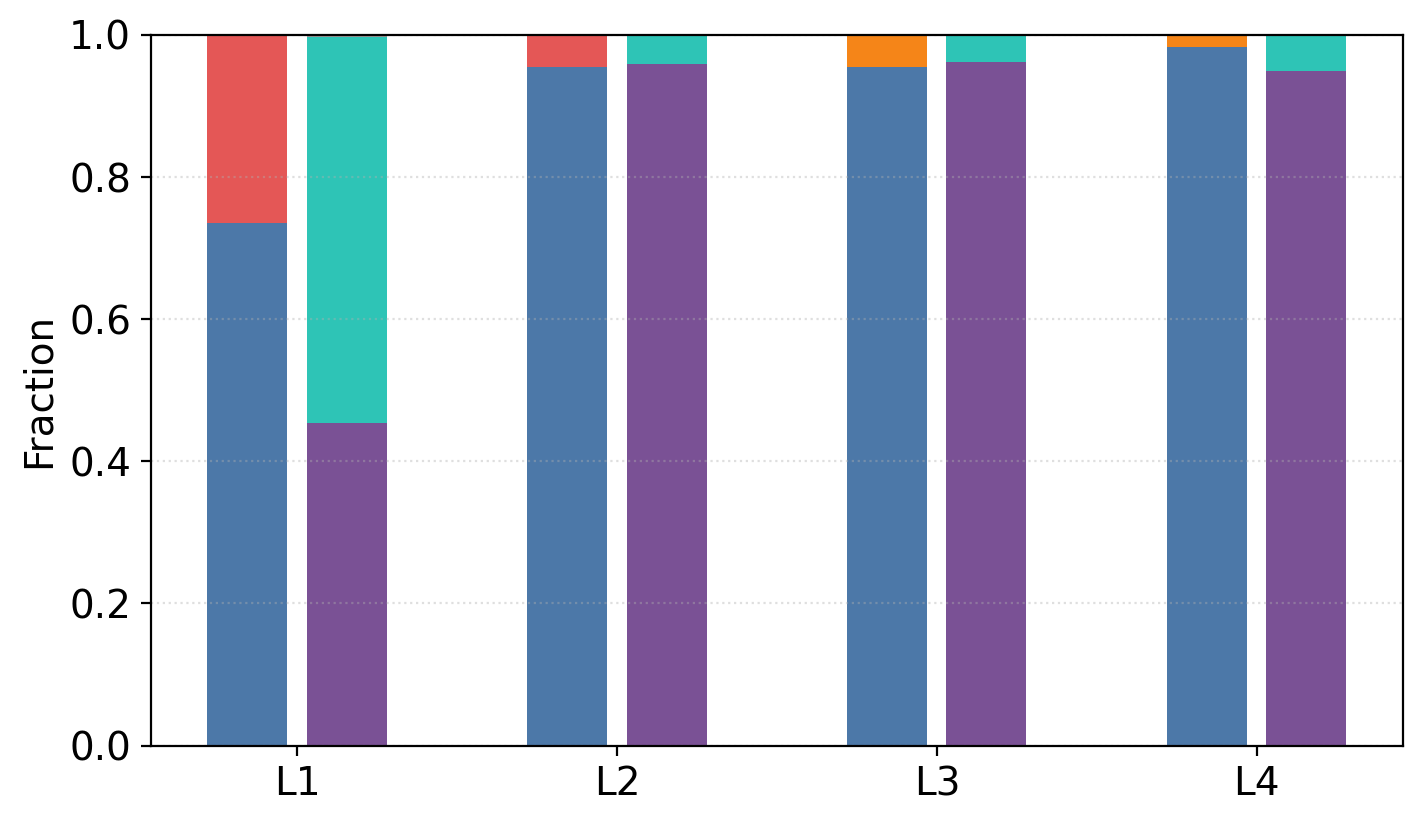}
    \caption{Covertype}
\end{subfigure}
\hfill
\begin{subfigure}[t]{0.33\linewidth}
    \centering
    \includegraphics[width=\linewidth]{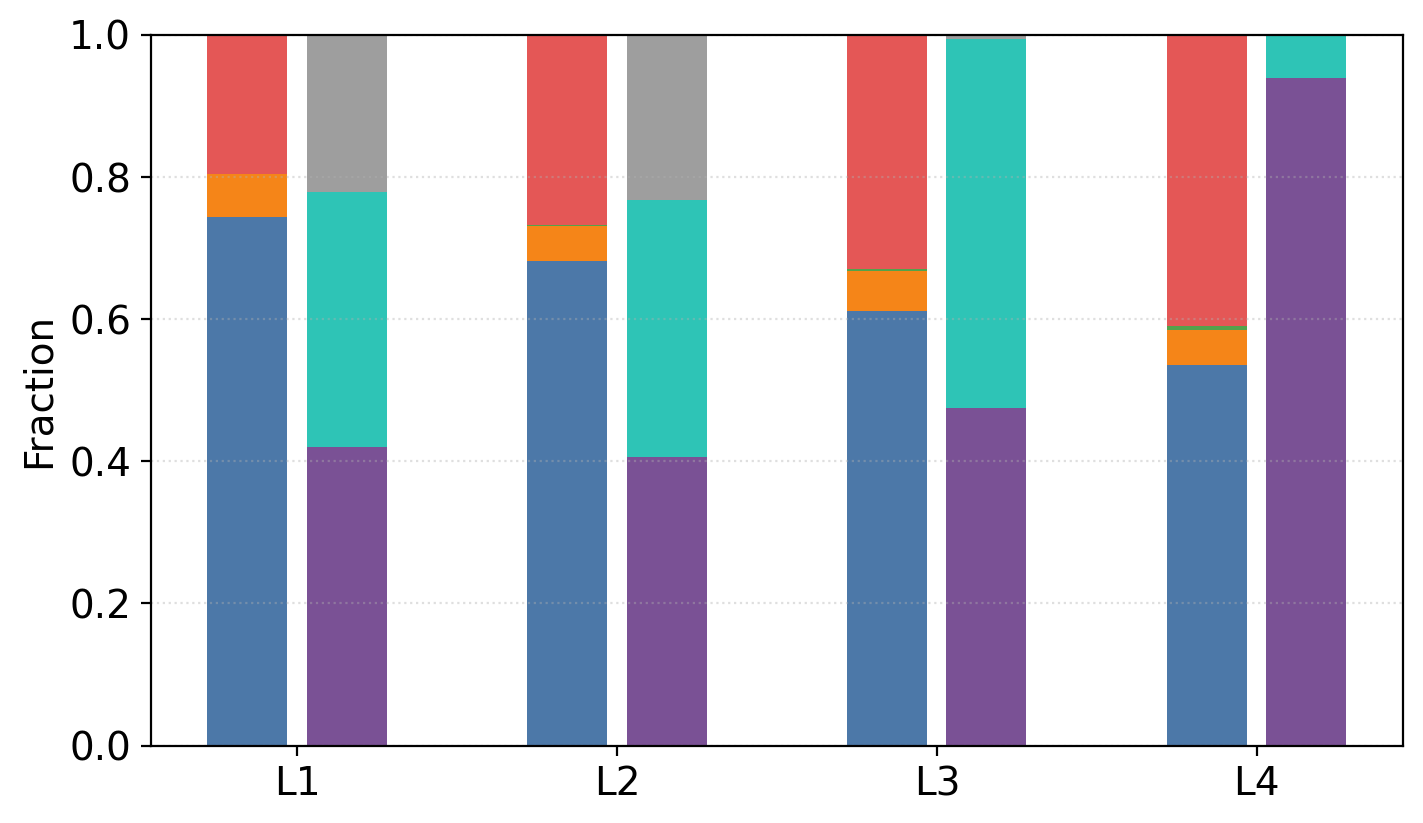}
    \caption{RSS}
\end{subfigure}
\hfill
\begin{subfigure}[t]{0.33\linewidth}
    \centering
    \includegraphics[width=\linewidth]{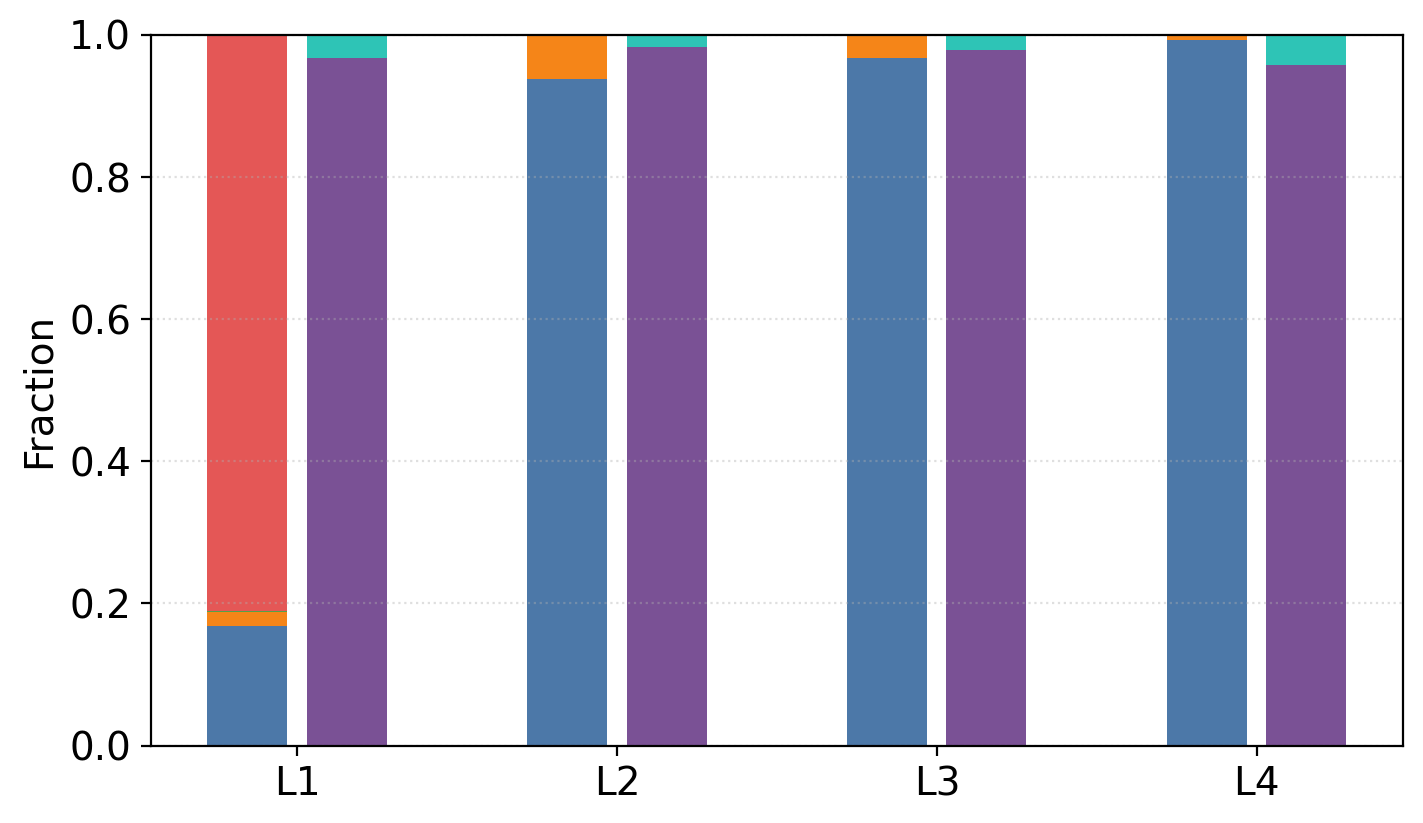}
    \caption{Higgs}
\end{subfigure}

\vspace{0.5em}

\begin{subfigure}[t]{0.65\linewidth}
    \centering
    \includegraphics[width=\linewidth]{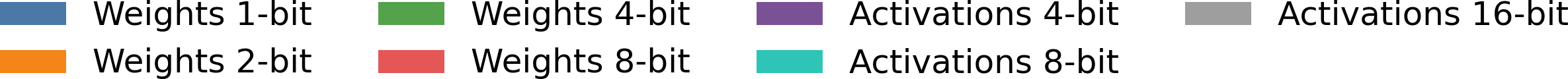}
\end{subfigure}

\caption{Per‑layer mixed‑precision distributions for weights and activations (fractions). For each layer $L_1-L4$, stacked bars shows the fraction of bit‑widths used for weights (left) and activations (right), highlighting how precision allocation differs across layers.}
\label{fig:Neuron_level_MP_MLP}
\end{figure}

\begin{table*}[!h]
\centering
\caption{Performance comparison across datasets under weight only quantization.}
\label{tab:main_results}
\scriptsize
\setlength{\tabcolsep}{3pt}
\renewcommand{\arraystretch}{1.4}

\begin{tabular}{c|c|cccc|c|c|c}
\toprule
\multirow{2}{*}{Dataset} & \multirow{2}{*}{Full Precision} & \multicolumn{4}{c|}{QAT} & \multirow{2}{*}{PPSO} & \multirow{2}{*}{OWQ} & \multirow{2}{*}{Ours} \\
\cline{3-6}
& & {1-bit} & {1.58-bit} & {4-bit} & {8-bit} & & & \\
\midrule
KVS (mse)$\downarrow$ & 5.379 & 2.292 $\pm$ 0.391 & 2.385 $\pm$ 0.062 & 2.619 $\pm$ 0.074 & 2.762 $\pm$ 0.396 & 2.498 $\pm$ 0.082 & 4.416 $\pm$ 2.392 & \textbf{2.270 $\pm$ 0.130} \\ \hline
VoD (mse)$\downarrow$ & 16.41 & 15.576 $\pm$ 0.117 & 14.854 $\pm$ 0.545 & 19.473 $\pm$ 0.353 & 19.605 $\pm$ 0.529 & 16.547 $\pm$ 0.507 & 19.888 $\pm$ 0.569 & \textbf{14.832 $\pm$ 0.164} \\ \hline
RSS (mse)$\downarrow$ & 17.62 & 8.882 $\pm$ 0.325 & 8.406 $\pm$ 0.470 & 13.347 $\pm$ 1.079 & 14.051 $\pm$ 2.732 & 8.487 $\pm$ 0.252 & 11.420 $\pm$ 0.369 & \textbf{8.317 $\pm$ 0.337}\\
\midrule
QoE (acc)$\uparrow$ & 0.7152 & 0.742 $\pm$ 0.023 & 0.749 $\pm$ 0.005 & 0.663 $\pm$ 0.034 & 0.643 $\pm$ 0.004 & 0.707 $\pm$ 0.032 & 0.522 $\pm$ 0.004 & \textbf{0.754 $\pm$ 0.02}  \\ \hline
Covertype (acc)$\uparrow$ & \textbf{0.9581} &  0.884 $\pm$ 0.005 & 0.915 $\pm$ 0.001 & 0.837 $\pm$ 0.01 & 0.854 $\pm$ 0.06 & 0.930 $\pm$ 0.002 & 0.822 $\pm$ 0.002 & \textit{0.933 $\pm$ 0.003} \\ \hline
Higgs (acc)$\uparrow$ & 0.7478 & 0.737 $\pm$ 0.001 & 0.742 $\pm$ 0.002 & 0.651 $\pm$ 0.003 & 0.695 $\pm$ 0.002 & 0.749$\pm$0.001 & 0.660 $\pm$ 0.007 & \textbf{0.767 $\pm$ 0.001} \\ \hline
QoE (f1)$\uparrow$ & 0.6205 & 0.680 $\pm$ 0.026 & 0.682 $\pm$ 0.018 & 0.393 $\pm$ 0.132 & 0.315 $\pm$ 0.004 & 0.634 $\pm$ 0.046 & 0.577 $\pm$ 0.029 & \textbf{0.665 $\pm$ 0.036}\\ \hline
Covertype (f1)$\uparrow$ & \textbf{0.937} & 0.794 $\pm$ 0.042 & 0.839 $\pm$ 0.013 & 0.739 $\pm$ 0.02 & 0.7310 $\pm$ 0.039 & 0.889 $\pm$ 0.001 & 0.722 $\pm$ 0.005 & \textit{0.897 $\pm$ 0.002} \\ \hline
Higgs (f1)$\uparrow$ & 0.7477 & 0.737 $\pm$ 0.001 & 0.742 $\pm$ 0.002 & 0.650 $\pm$ 0.002 & 0.695 $\pm$ 0.002 & 0.749$\pm$0.001 & 0.663 $\pm$ 0.006 & \textbf{0.767 $\pm$ 0.001}\\

\bottomrule
\end{tabular}
\end{table*}

\begin{figure}[t]
\centering
\begin{subfigure}[t]{0.4\linewidth}
    \centering
    \includegraphics[width=\linewidth]{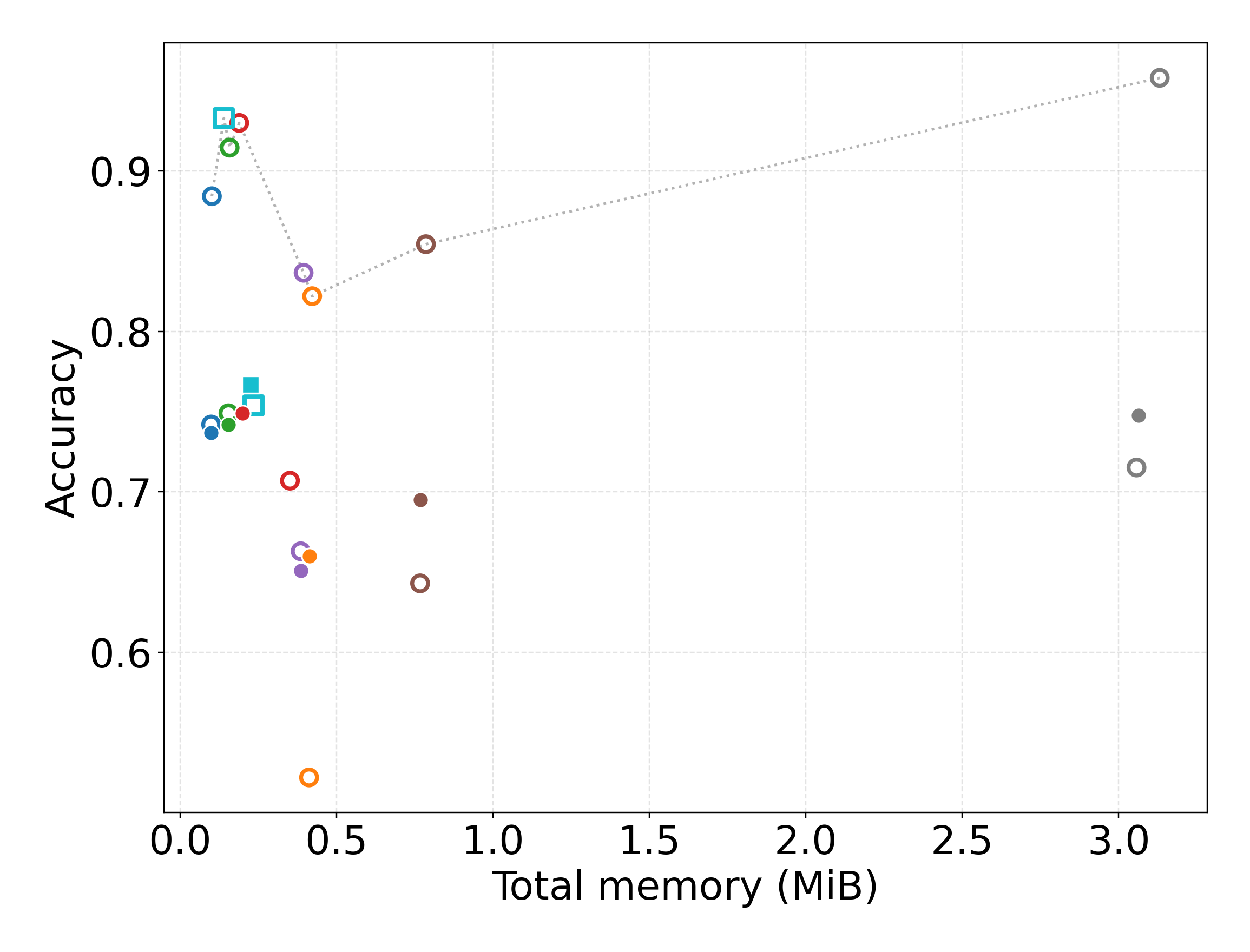}
    \caption{ \small Classification w-o quantization}
\end{subfigure}
\begin{subfigure}[t]{0.4\linewidth}
    \centering
    \includegraphics[width=\linewidth]{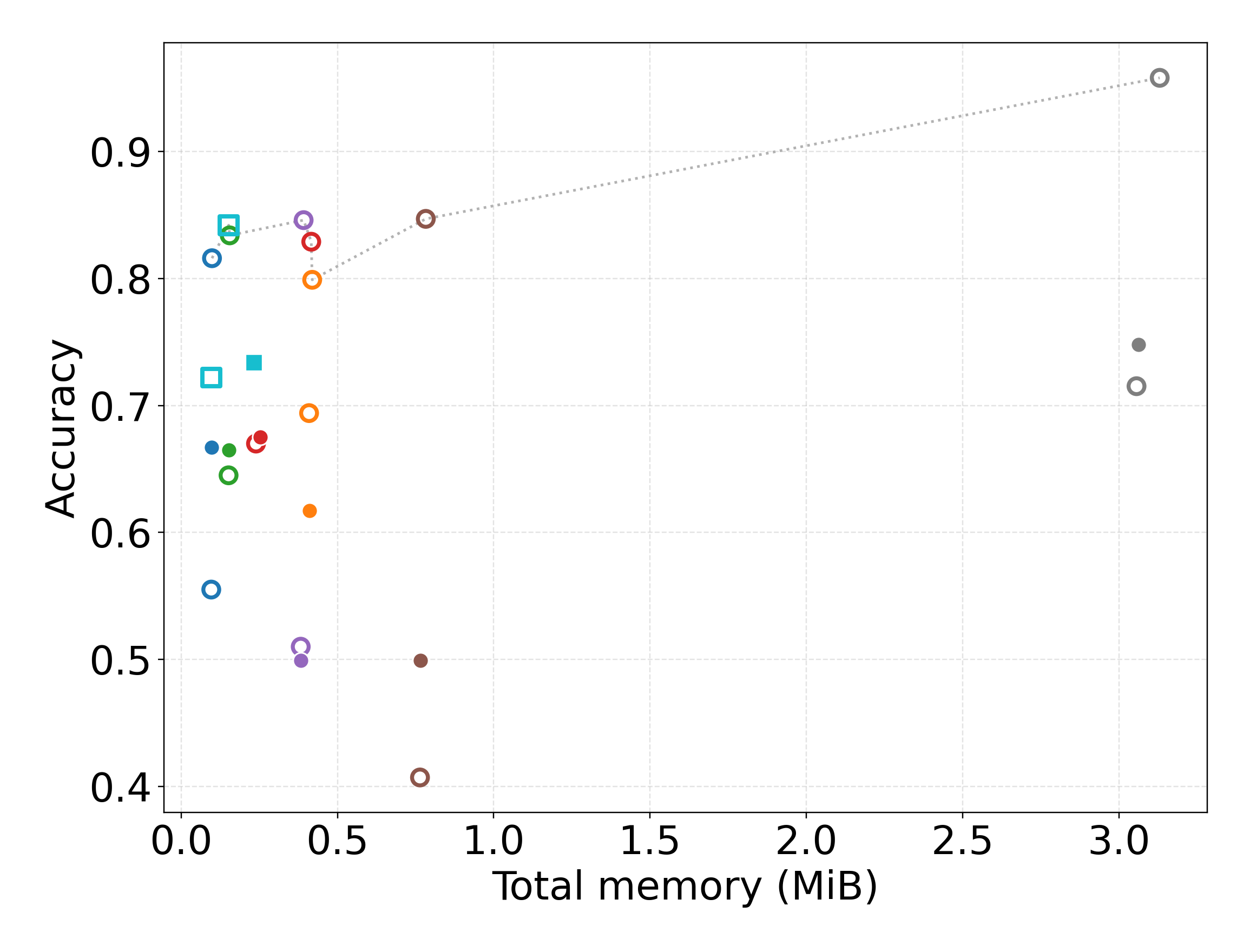}
    \caption{Classification w+act quantization}
\end{subfigure}

\begin{subfigure}[t]{0.4\linewidth}
    \centering
    \includegraphics[width=\linewidth]{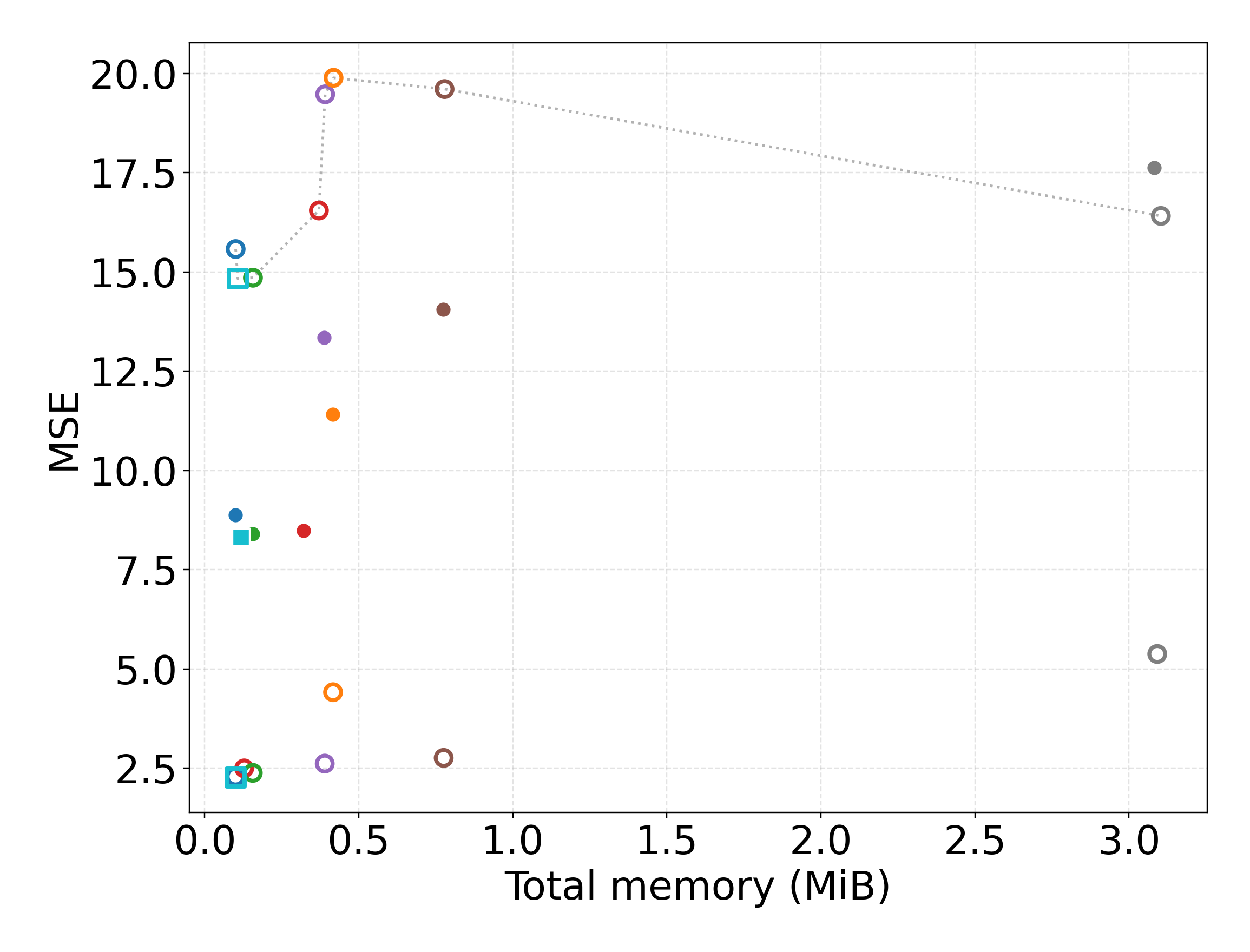}
    \caption{Regression w-o quantization}
\end{subfigure}
\begin{subfigure}[t]{0.4\linewidth}
    \centering
    \includegraphics[width=\linewidth]{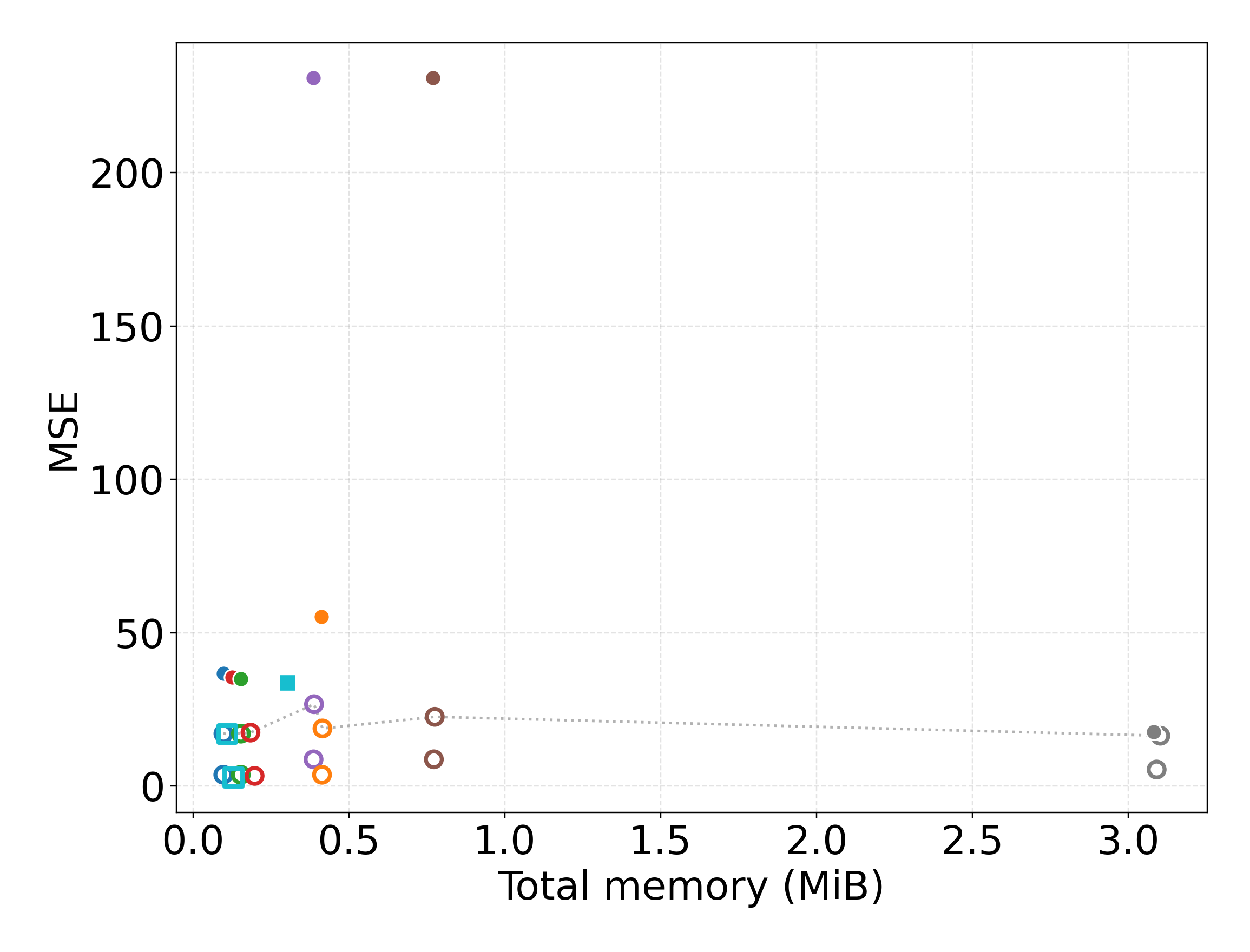}
    \caption{Regression w+act quantization}
\end{subfigure}

\vspace{0.5em}

\begin{subfigure}[t]{0.48\linewidth}
    \centering
    \includegraphics[width=\linewidth]{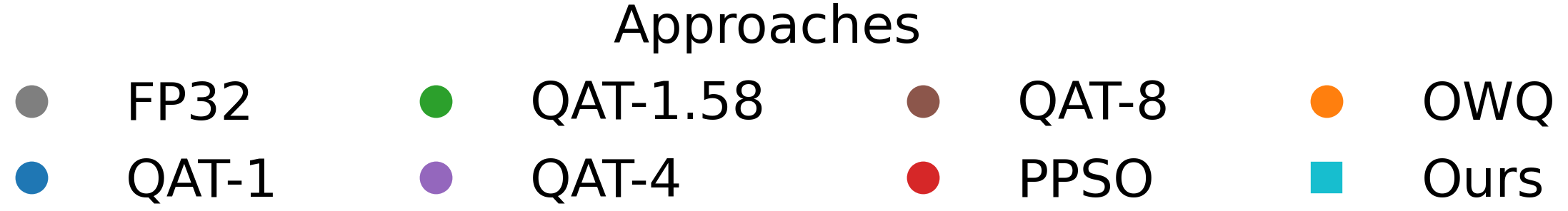}
\end{subfigure}

\caption{Memory-utility trade-off for a 4×512 MLP. (a) Weights-only and (b) weights+activations quantization. Each point represents one method (FP32, QAT-1/1.58/4/8, PPSO, OWQ, Ours); Ours is marked by a square. The x-axis shows total theoretical memory (MiB) and the y-axis shows utility (accuracy or MSE). NMP-QAT consistently occupies the low-memory/high-utility region.}
\label{fig:theoretical_memory_footprint_MLP}
\end{figure}

\begin{figure}[!ht]
\centering
\begin{subfigure}[t]{0.4\linewidth}
    \centering
    \includegraphics[width=\linewidth]{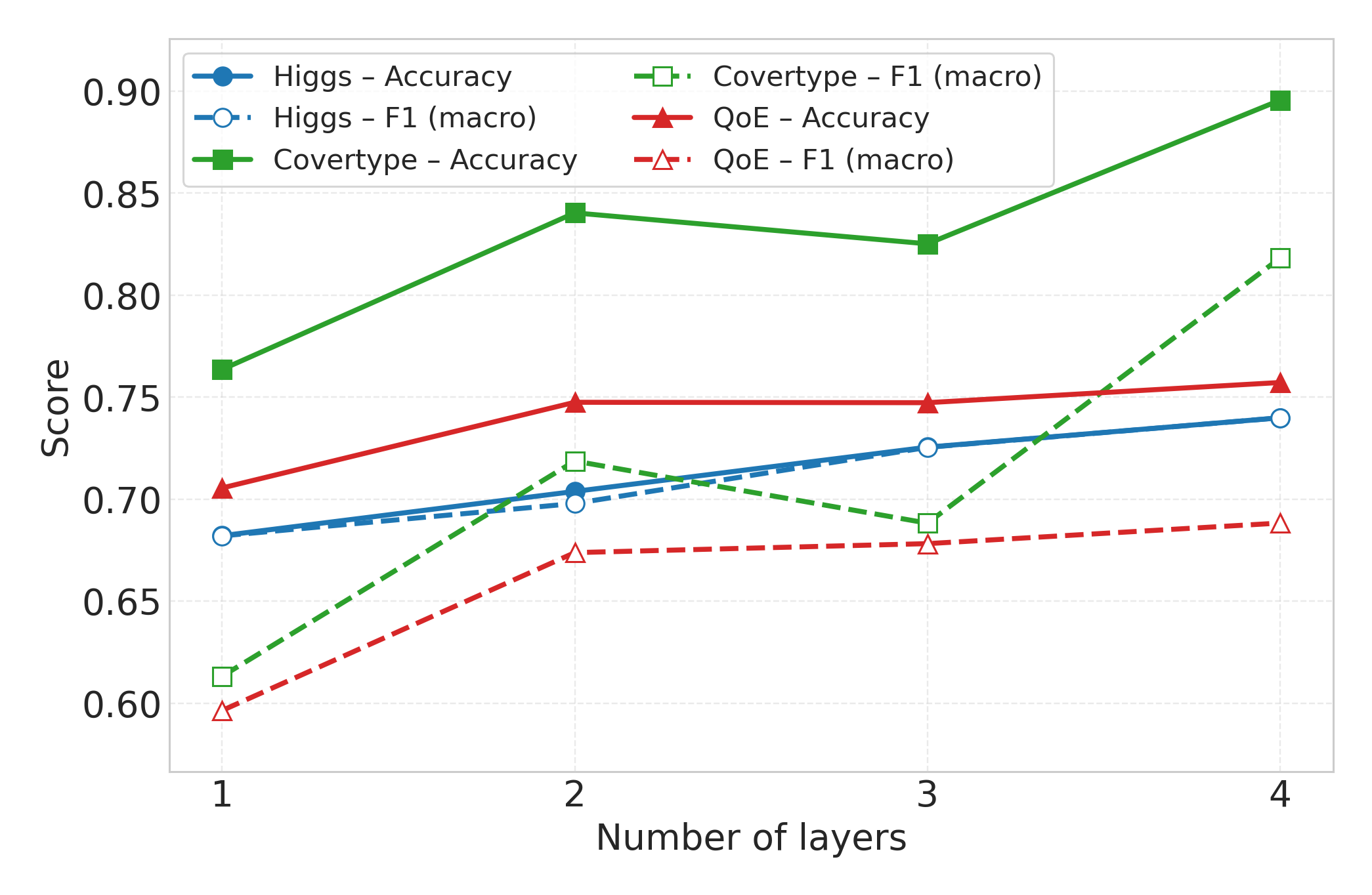}
    \caption{}
\end{subfigure}
\begin{subfigure}[t]{0.4\linewidth}
    \centering
    \includegraphics[width=\linewidth]{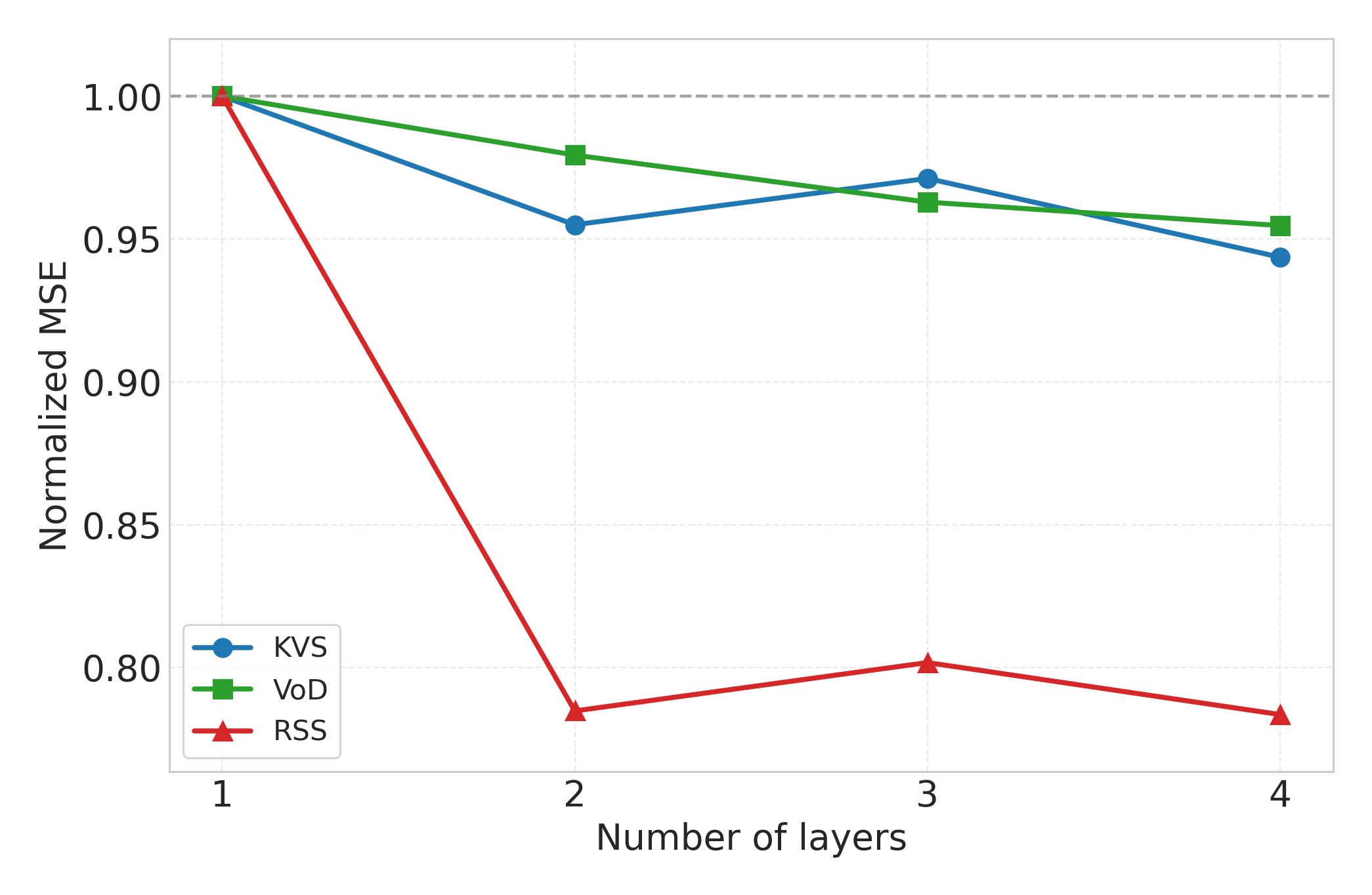}
    \caption{}
\end{subfigure}

\caption{Effect of model depth on NMP-QAT performance. (a) Accuracy and F1-score vs. number of layers for classification datasets. (b) Normalized MSE (relative to 1-layer baseline) for regression datasets. Deeper models consistently improve predictive quality.}
\label{fig:layer_impact_NMP_QAT}
\end{figure}

\begin{table*}[!ht]
\centering  
\caption{Comparison of wall-clock (in min:sec format) runtime with Tabformer model for mixed-precision approaches.}
\label{tab:comp_runtime}
\begin{tabular}{c|c|c|c|c|c|c}
\toprule
Approach & KVS & VoD & RSS & QoE & Covertype & Higgs  \\
\midrule
OBQ & 0:50 & 1:55 & 3:58 & 0:10 & 40:37 & 67:01 \\ \hline
PPSO & 189:53 & 257:19 & 406:32 & 120:13 & 1222:45 & 3840:22 \\ \hline
Ours & 4:45 & 6:25 & 22:04 & 1:35 & 72:57 & 90:09 \\
\bottomrule
\end{tabular}
\end{table*}

\end{document}